%% file: main.tex
\documentclass{article}

\usepackage{microtype}
\usepackage{graphicx}
\usepackage{booktabs} %

\usepackage{hyperref}

\usepackage[accepted]{icml2024}

\usepackage{amsmath}
\usepackage{amssymb}
\usepackage{mathtools}
\usepackage{amsthm}

\usepackage[capitalize,noabbrev]{cleveref}

\makeatletter
\newtheorem*{rep@theorem}{\rep@title}
\newcommand{\newreptheorem}[2]{%
	\newenvironment{rep#1}[1]{%
		\def\rep@title{#2 \ref{##1}}%
		\begin{rep@theorem}}%
		{\end{rep@theorem}}}
\makeatother

\theoremstyle{plain}
\newtheorem{theorem}{Theorem}[section]
\newreptheorem{theorem}{Theorem}
\newtheorem{proposition}[theorem]{Proposition}
\newreptheorem{proposition}{Proposition}
\newtheorem{lemma}[theorem]{Lemma}
\newreptheorem{lemma}{Lemma}
\newtheorem{corollary}[theorem]{Corollary}
\newreptheorem{corollary}{Corollary}
\theoremstyle{definition}
\newtheorem{definition}[theorem]{Definition}
\newreptheorem{definition}{Definition}

\theoremstyle{remark}

\definecolor{darkgreen}{RGB}{0,192,0}
\definecolor{darkred}{RGB}{192,0,0}

\usepackage{pifont}
\definecolor{myred}{RGB}{215,48,39}
\definecolor{mygreen}{RGB}{26,152,80}
\newcommand{\cmark}{\textcolor{mygreen}{\ding{51}}}
\newcommand{\xmark}{\textcolor{myred}{\ding{55}}}

\usepackage[textsize=tiny]{todonotes}

\let\oldtodo\todo
\renewcommand{\todo}[1]{\PackageWarning{mystyle}{#1}\oldtodo{#1}}

\input{symbols}

\usepackage{enumitem}

\icmltitlerunning{On the Universality of Volume-Preserving and Coupling-Based Normalizing Flows}

\begin{document}

\twocolumn[
\icmltitle{On the Universality of Volume-Preserving and \texorpdfstring{\\}{} Coupling-Based Normalizing Flows}

\icmlsetsymbol{equal}{*}

\begin{icmlauthorlist}
\icmlauthor{Felix Draxler}{hd}
\icmlauthor{Stefan Wahl}{hd}
\icmlauthor{Christoph Schnörr}{hd}
\icmlauthor{Ullrich Köthe}{hd}
\end{icmlauthorlist}

\icmlaffiliation{hd}{Heidelberg University, Germany}
\icmlcorrespondingauthor{Felix Draxler}{felix.draxler@iwr.uni-heidelberg.de}

\icmlkeywords{Machine Learning, ICML}

\vskip 0.3in
]

\printAffiliationsAndNotice{}  %

\begin{abstract}
    We present a novel theoretical framework for understanding the expressive power of normalizing flows. Despite their prevalence in scientific applications, a comprehensive understanding of flows remains elusive due to their restricted architectures. Existing theorems fall short as they require the use of arbitrarily ill-conditioned neural networks, limiting practical applicability. We propose a distributional universality theorem for well-conditioned coupling-based normalizing flows such as RealNVP \cite{dinh2017density}. In addition, we show that volume-preserving normalizing flows are not universal, what distribution they learn instead, and how to fix their expressivity. Our results support the general wisdom that affine and related couplings are expressive and in general outperform volume-preserving flows, bridging a gap between empirical results and theoretical understanding.
\end{abstract}

\section{Introduction}

Density estimation and generative modeling of complex distributions is a fundamental problem in statistics and machine learning, with applications ranging from computer vision \cite{kingma2018glow} to thermodynamic systems \cite{noe2019boltzmann, albergo2019flowbased, nicoli2021estimation} and uncertainty quantification \cite{ardizzone2018analyzing}.

\begin{figure}[th]
    \centering
    \includegraphics[width=\linewidth]{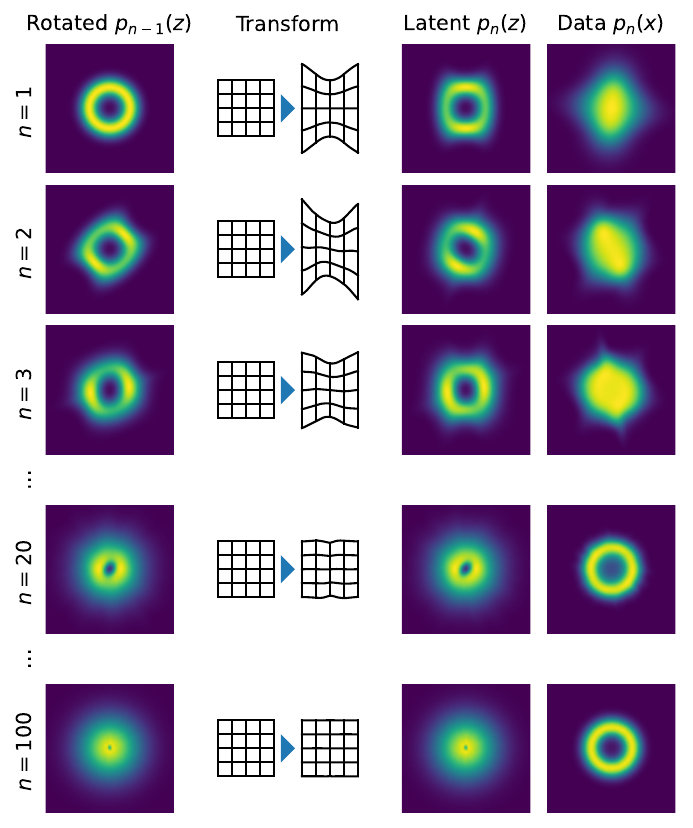}
    \caption{
        \textbf{Our universality proof constructs a normalizing flow by iteratively adding affine coupling blocks.} We illustrate this by constructing such a flow from topologically challenging toy data. Each block first rotates the distribution $p_{n-1}(z)$ from the previous step \textit{(first column)}, then applies an affine coupling layer that transforms the active dimensions to zero mean and unit variance for each passive coordinate $b$ \textit{(second column)}. The resulting latent distribution converges step by step \textit{(third column)} to a standard normal distribution, where the learned additional layers essentially learn the identity \textit{(last row)}. The data distribution $p_\theta(x)$ converges in parallel \textit{(right)}.
    }
    \label{fig:universality-example}
\end{figure}

Normalizing flows are a class of generative models that learn a probability density $p_\theta(x)$ from samples $x \sim p(x)$ or from a potentially unnormalized ground truth density $\hat p(x) \propto p(x)$. They are implemented by transporting a simple multivariate base density such as the standard normal via a learned invertible function to the distribution of interest.

Constructing flexible and tractable invertible neural networks is nontrivial and a significant body of work has developed a plethora of architectures, see \cite{kobyzev2021normalizing} for an overview. The evidence for which architecture to choose in practice is mostly limited to empirical results, however. In this work, we prove rigorous results regarding the universality of the families of volume-preserving and coupling-based normalizing flows.

First, we consider volume-preserving flows such as \cite{dinh2015nice, sorrenson2019disentanglement, toth2020hamiltonian}. Volume-preservation can be useful in certain applications, such as disentanglement \cite{sorrenson2019disentanglement} or learning distributions with controllable temperature $p(x|T) \propto (\hat p(x))^T$ \cite{dibak2022temperature}. However, we show that they are not universal and learn a biased distribution in practice. We also provide a simple solution to restore universality, by adding a single one-dimensional non-volume-preserving layer.

Second, we improve on the universality theory of coupling-based normalizing flows (not preserving volume). These flows are a particularly efficient variant of parameterizing invertible neural networks, with fast training and inference. Despite their seemingly strong architectural constraints, in practice even the simple affine coupling-based normalizing flow \cite{dinh2017density} can learn high-dimensional distributions such as images \cite{kingma2018glow}.

Theoretical explanations for this architecture's ability to fit complex distributions are limited. Existing proofs make assumptions that are not valid in practice, as the involved constructions rely on ill-conditioned neural networks \cite{teshima2020couplingbased, koehler2021representational} or construct a volume-preserving flow \cite{koehler2021representational}. We introduce a new proof for the distributional universality of coupling-based normalizing flows that does not require ill-conditioned neural networks to converge. This proof is constructive, showing that training affine coupling blocks sequentially converges to the correct target (compare \cref{fig:universality-example}).

In summary, we contribute:
\begin{itemize}[itemsep=1pt,topsep=0pt]
    \item We show for the first time that volume-preserving flows are not universal, derive what distribution they converge to instead and provide simple fixes for their shortcomings in \cref{sec:limitations-vol-preserving-flows}.
    \item We then show that the most recent distributional universality proof for affine coupling-based normalizing flows by \citet{koehler2021representational} constructs such a volume-preserving flow in \cref{sec:shortcomings}.
    \item Finally, we give a new universality proof for coupling-based normalizing flows that is not volume-preserving, considers the full support of the distribution, and is not ill-conditioned in \cref{sec:affine-coupling-universality}.
\end{itemize}
Our results validate insights previously observed only empirically: Affine coupling blocks are an effective foundation for normalizing flows, and volume-preserving flows have limited expressive power. %

\newpage

\section{Related Work}

\begin{table}
	\caption{\textbf{Our construction of a universal coupling flow overcomes two important limitations} of the previous work on arbitrary input $p(x)$ \citep{teshima2020couplingbased, koehler2021representational}: We use well-conditioned coupling blocks and allow variable volume change $|f_\theta'(x)| \neq \operatorname{const}$, which is necessary for universality in KL divergence by \cref{thm:volume-preserving-not-universal}.}
	\label{tab:improvements}
		\begin{center}
	\resizebox{\linewidth}{!}{
\begin{tabular}{lccc}
\toprule
& Teshima et al. & Koehler et al. & Thm.~\ref{thm:affine-coupling-universality} \\ \midrule
Well-conditioned & \xmark & \xmark & \cmark \\
Variable $|f_\theta'(x)|$ & \cmark & \xmark & \cmark \\
Global support & \xmark & \xmark & \cmark \\
\bottomrule
\end{tabular}
   }
		\end{center}
\end{table}

Normalizing flows are a class of generative models based on invertible neural networks \cite{rezende2015variational}. We focus on the analytical expressivity of volume-preserving and coupling-based normalizing flows, see \cref{sec:background}.

That coupling-based normalizing flows work well in practice despite their restricted architecture has sparked the interest of several papers analyzing their distributional universality, i.e.~the question whether they can approximate any target distribution to arbitrary precision (see \cref{def:distributional-universality}). 
\citet{teshima2020couplingbased} showed that coupling flows are universal approximators for invertible functions, which results in distributional universality.
\citet{koehler2021representational} demonstrated that affine coupling-based normalizing flows can approximate any distribution with arbitrary precision using just three coupling blocks.
However, these works rely on ill-conditioned coupling blocks and consider convergence only on a bounded subspace. Our work addresses these limitations and shows that training a normalizing flow layer by layer yields universality.
In addition, we find that \citet{koehler2021representational} effectively construct a volume-preserving flow, which we show not to be universal in KL divergence, the loss used in practice.
This line of work is summarized in \cref{tab:improvements}.

Some works show distributional universality of \textit{augmented} affine coupling-based normalizing flows, which add at least one additional dimension usually filled with exact zeros \cite{huang2020augmented, koehler2021representational,lyu2022universality}. The problem with adding additional zeros is that the flow is not exactly invertible anymore in the data domain and usually loses tractability of the change of variables formula (\cref{eq:change-of-variables}). \citet{lee2021universala} add i.i.d.~Gaussians as additional dimensions, which again allows density estimation, but they only show how to approximate the limited class of log-concave distributions. Our universality proof does not rely on such a construction.

Other theoretical work on the expressivity of normalizing flows considers more expressive invertible neural networks, including SoS polynomial flows, Neural ODEs and Residual Neural Networks \cite{jaini2019sumofsquares, zhang2020approximation, teshima2020universal, ishikawa2022universal}. Another line of work found that the number of required coupling blocks is independent of dimension $D$ for Gaussian distributions compared to $O(D)$ Gaussianization blocks that lack couplings between dimensions \cite{koehler2021representational, draxler2022whitening, draxler2023convergence}.

Regarding volume-preserving flows, our results complement the fact that Hamiltonian Monte Carlo (HMC) resamples momenta at every step in order to sample the correct target distribution \cite{neal2011mcmc}.

\section{Background}
\label{sec:background}

\textbf{Normalizing Flows} are a class of generative models that represent a distribution $p_\theta(x)$ with parameters $\theta$ by learning an invertible function $z = f_\theta(x)$ so that the latent codes $z \in \RR^D$ obtained from the data $x \in \RR^D$ are distributed like a standard normal distribution $p(z) = \Nn(z; 0, I)$. Via the change of variables formula, see \cite{kothe2023review} for a review, this invertible function yields an explicit form for the density $p_\theta(x)$:
\begin{equation}
    \label{eq:change-of-variables}
    p_\theta(x) = p(z=f_\theta(x)) |f_\theta'(x)|,
\end{equation}
where $f_\theta'(x) = \frac{\partial}{\partial x}f_\theta(x)$ is the Jacobian matrix of $f_\theta$ at $x$ and $|f_\theta'(x)|$ is its absolute determinant. Note that we consistently denote by $p(\cdot)$ a ground truth target distribution and by $p_\theta(\cdot)$ a learned approximation.

\Cref{eq:change-of-variables} allows easily evaluating the model density at points of interest. Obtaining samples from $p_\theta(x)$ can be achieved by sampling from the latent standard normal and applying the inverse $f_\theta^{-1}(z)$ of the learned transformation:
\begin{equation}
    \label{eq:sampling}
    x = f_\theta^{-1}(z) \sim p_\theta(x) \text{ for } z \sim p(z).
\end{equation}

The change of variables formula (\cref{eq:change-of-variables}) can be used directly to train a normalizing flow. The corresponding loss minimizes the Kullback-Leibler divergence between the true data distribution $p(x)$ and the learned distribution, which can be optimized via a Monte-Carlo estimate of the involved expectation:
\begin{align}
    \label{eq:loss}
    \Ll &
    = \KL{p(x)}{p_\theta(x)} \\&
    = \EE_{x \sim p(x)}[\log p(x) - \log p_\theta(x)] \\&
    = \EE_{x \sim p(x)}[-\log p_\theta(x)] + \text{const.}
    \label{eq:negative-log-likelihood}
\end{align}
This last variant makes clear that minimizing this loss is exactly the same as maximizing the log-likelihood of the training data. For training, the expectation value is approximated using (batches of) training samples $x_{1, \dots, N}$.

In order for \cref{eq:change-of-variables,eq:sampling} to be useful in practice, $f_\theta(x)$ must have (i) a tractable inverse $f_\theta^{-1}(z)$ for fast sampling, and (ii) a tractable Jacobian determinant $|f_\theta'(x)|$ for fast training while (iii) being expressive enough to model complicated distributions. These constraints are nontrivial to fulfill at the same time and significant work has been put into constructing such invertible neural networks, see \cite{kobyzev2021normalizing}.

\textbf{Coupling-Based Normalizing Flows} \label{sec:coupling-flows} are an invertible neural network design that lies in a sweet spot of scaling well to large dimensions yet remaining fast to sample from \cite{draxler2023convergence} and exhibits a tractable Jacobian determinant. Its basic building block is the coupling layer, which consists of one invertible function $\tilde x_i = c(x_i; \psi_i)$ for each dimension, but with a twist: Only the second half of the dimensions $a = x_{D/2 + 1, \dots, D}$ (\textit{active}) is changed in a coupling layer, and the parameters of this transformation $\psi = \psi(b)$ are predicted by a function called the \textit{conditioner} that takes the first half of dimensions $b = x_{1, \dots, D/2}$ (\textit{passive}) as input:
\begin{equation}
    \label{eq:coupling-layer}
    \tilde x_i = (f_\text{cpl}(x))_i = \begin{cases}
        b_i & i \leq D/2 \\
        c(a_{i-D/2}; \psi_{i-D/2}(b)) & \text{else.}
    \end{cases}
\end{equation}
In practice, the conditioner of each block is realized by a neural network $\psi_\phi(b)$ with parameters $\phi$. Calculating the inverse of the coupling layer is easy, as $b = \tilde b$ for the passive dimensions. This allows computing the parameters $\psi(b)$ necessary to invert the active half of dimensions:
\begin{equation}
    \label{eq:coupling-layer-inverse}
    x_i = (f_\text{cpl}^{-1}(\tilde x))_i = \begin{cases}
        \tilde b_i & i \leq D/2 \\
        c^{-1}(\tilde a_{i-D/2}; \psi_{i-D/2}(\tilde b)) & \text{else.}
    \end{cases}
\end{equation}

Choosing the right one-dimensional invertible function $c(x; \psi)$ is the subject of active research, see our list in \cref{app:compatible-couplings} and \citet{kobyzev2021normalizing}. Many applications use affine-linear functions \cite{dinh2017density}:
\begin{align}
    \label{eq:affine-coupling}
    c(x; s, t) = sx + t,
\end{align}
where $s > 0$ and $t$ are predicted by the conditioner $\psi(b)$ as a function of the passive dimensions. Especially for smaller-dimensional problems it has proven useful to use more flexible $c$ such as rational-quadratic splines \cite{durkan2019neural}. The positive and negative universality theorems in this paper apply to all coupling architectures we are aware of. At the same time, we give a direct reason for using more expressive couplings, as they can learn the same distributions with fewer layers in \cref{sec:expressive-universality}.

In order to be expressive, a coupling flow consists of a stack of coupling layers, each with a different active and passive subspace. Varying the subspaces is realized by an additional layer before each coupling which mixes dimensions by a rotation matrix $Q \in SO(D)$:
\begin{align}
    f_\text{rot}(x) = Qx, \qquad
    f_\text{rot}^{-1}(\tilde x) = Q^T \tilde x.
\end{align}
The matrix $Q$ can either be chosen to be a hard permutation or any matrix with orthonormal columns, which mixes dimensions linearly. The latter can optionally be learned \cite{kingma2018glow}. From a representational perspective, these variants are interchangeable because any soft permutation can be represented by a constant number of coupling blocks with hard permutations \cite{koehler2021representational}.  %

A rotation layer together with a coupling layer forms a coupling block:
\begin{equation}
    \label{eq:coupling-block}
    f_\text{blk}(x) = (f_\text{cpl} \circ f_\text{rot})(x) = f_\text{cpl}(Qx).
\end{equation}
In the \cref{sec:coupling-universality}, we are concerned with what distributions $p(x)$ a potentially deep concatenation of coupling blocks can represent. Taken together, the parameters of the conditioners of the coupling layers $\phi$ together with the rotation matrices $Q$ make up the learnable parameters $\theta$ of the entire normalizing flow.

\textbf{Volume-Preserving Normalizing Flows} \label{sec:vp-flows} or sometimes \textit{incompressible flows} are a variant of normalizing flows that have a constant Jacobian determinant $|f_\theta'(x)| = \operatorname{const}$. This simplifies the change of variables formula in \cref{eq:change-of-variables} above, where $C = |f_\theta'(x)|$:
\begin{equation}
    \label{eq:incompressible-change-of-variables}
    p_\theta(x) = p(z = f_\theta(x)) C.
\end{equation}

Volume-preserving flows have been demonstrated to have useful properties in certain applications such as disentanglement \cite{sorrenson2019disentanglement}, or temperature-scaling in Boltzmann generators \cite{dibak2022temperature} or to preserve volume in physical state-space \cite{toth2020hamiltonian}. However, we show that the volume-preserving change of variables does not allow for universal normalizing flows in \cref{eq:incompressible-change-of-variables} regardless of the architecture in \cref{sec:limitations-vol-preserving-flows}.

For one-dimensional functions, a constant volume change implies that $f_\theta(x) = Cx+t$ is linear. For multivariate functions, $f_\theta(x)$ can be nonlinear, only that any volume change in one dimension must be compensated by an inverse volume change in the remaining dimensions.  Prominent implementations are nonlinear independent components estimation (NICE) \cite{dinh2015nice}, general incompressible-flow networks (GIN) \cite{sorrenson2019disentanglement} or Neural Hamiltonian Flows (NHF) \cite{toth2020hamiltonian}. For example, GIN realizes volume-preserving coupling blocks by requiring that $\sum_{i=1}^{D/2} \log s_i(b) = \operatorname{const}$. We list common volume-preserving constructions in \cref{app:volume-preserving-architectures}.

\textbf{Distributional Universality}
\label{sec:distributional-universality}
means that a certain class of generative models can represent any distribution $p(x)$. Due to the nature of neural networks, we cannot hope for our generative model to \emph{exactly} (i.e.~exact equality in the mathematical sense) represent $p(x)$. This becomes clear via an analogue in the context of regression: A neural network with ReLU activations always models piecewise linear functions, and as such it can never \emph{exactly} regress a parabola $y = x^2$. However, for every finite value of $\epsilon > 0$ and given more and more linear pieces, it can follow the parabola ever so closer, so that the average distance between $x^2$ and $f_\theta(x)$ vanishes: $\EE_{x \sim p(x)}[|x^2 - f_\theta(x)|^2] < \epsilon$. To characterize the expressivity of a class of neural networks, it is thus instructive to call a class of networks universal if the error between the model and any target can be reduced arbitrarily.

In terms of representing distributions $p(x)$, the following definition captures universality of a class of model distributions, similar to \citep[Definition 3]{teshima2020couplingbased}:
\begin{definition}
    \label{def:distributional-universality}
    A set of probability distributions $ \Pp $ is called a \emph{distributional universal approximator} if for every possible target distribution $p(x)$ there is a sequence of distributions $p_n(x) \in \Pp$ such that $ p_n(x) \xrightarrow{n \to \infty} p(x) $.
\end{definition}
The formulation of universality as a convergent series is useful as it (i) captures that the distribution in question $p(x)$ may not lie in $\Pp$, and (ii) the series index $n$ usually reflects a hyperparameter of the underlying model corresponding to computational requirements (for example, the depth of the network).

We have left the exact definition of the limit ``$p_n(x) \xrightarrow{n \to \infty} p(x)$'' open as we may want to consider different variations of convergence. The existing literature on affine coupling-based normalizing flows considers weak convergence \cite{teshima2020couplingbased} respectively convergence in Wasserstein distance \cite{koehler2021representational}. %
Many metrics of convergence have been proposed, see \citet{gibbs2002choosing} for a systematic overview.

While we consistently state our assumptions, we usually restrict ourselves to data distributions with densities $p(x)$ that are bounded and continuous, and that have infinite support and finite moments, which covers distributions of practical interest.

\section{Non-Universality of Volume-Preserving Flows}
\label{sec:limitations-vol-preserving-flows}

In this section, we consider normalizing flows with constant Jacobian determinant $|f_\theta'(x)| = \operatorname{const}$. We show that \textbf{volume-preserving flows are not universal in KL divergence}. We then propose how universality can be recovered.

From the volume-preserving change-of-variables in \cref{eq:incompressible-change-of-variables} we can derive the KL divergence $\KL{p(x)}{p_\theta(x)}$ in the special case of $C=1$:
\begin{align}
    \Ll &
    = \int p(x) \log \frac{p(x)}{p_\theta(x)} dx \\&
    = -H[p(x)] - \int p(x) \log p(z=f_\theta(x)) dx.
\end{align}
Only the last term depends on $f_\theta$. To derive the minimizer, consider the data $p(x)$ and latent distribution $p(z)$ on a regular grid over $\RR^D$ with some spacing $a > 0$. Then, define a volume-preserving flow with $C=1$ that permutes the grid cells $B_i \mapsto B_{f(i)}$ (within the cells, keep the relative positions). Then, discretize the above integral on the grid by approximating the latent probability by the average density in each cell, that is $p(z) \approx \frac{1}{a^D} p(B_i: z \in B_i)$:
\begin{align}
    &- \int p(x) \log p(z=f_\theta(x)) dx \\
    &\approx - \sum_{i} p(x \in B_{s_x(i)}) \log(p(z \in B_{f(i)}) / |a^D|).
\end{align}
This is minimized by a bijective $f^*: \NN \to \NN$ that permutes the grid cells such that the cell with the highest probability $p(x \in B_i)$ in the data space aligns with the cell with the highest (logarithmic) probability in latent space, and so on:
\begin{equation}
    f^*(i) = s_z(s_x^{-1}(i)),
    \label{eq:vp-grid-sorted}
\end{equation}
where $s_{v}(i)$ is a sorting of the grid cells, determined by probability mass $p(v \in B_i)$ for $v = x$ respectively $v = z$.

The following theorem makes the above argument continuous and determines the optimal volume change $C > 0$:
\begin{theorem}
    \label{thm:volume-preserving-lower-bound}
    Given a continuous bounded input density $p(x)$. Then, for any volume-preserving flow $p_\theta(x)$ with a standard normal latent distribution, the achievable KL divergence is bounded from below:
    \begin{equation}
        \KL{p(x)}{p_\theta(x)} \geq \KL{p^*(z)}{\Nn(0, |\Sigma_{p^*(z)}|^{\frac{1}{D}} I)},
    \end{equation}
    where $p^*(z)$ is constructed by decreasingly sorting the probability densities $p(x)$ from the origin with unit volume change, and $\Sigma_{p^*(z)}$ is its covariance matrix. The minimal loss is achieved for $C = |\Sigma_{p^*(z)}|^{-\frac{1}{2D}}$.
\end{theorem}
The optimal $p^*(x)$ and its latent counterpart $p^*(z)$ are constructed by sorting both the data and latent space by density and progressively assigning regions of decreasing density to each other (see proof in \cref{app:volume-preserving-minimizer}). \Cref{fig:effet-const-jac-example} shows how this optimal distribution $p^*(x)$ differs from the target $p(x)$ for a bimodal toy distribution in 2D.

\begin{figure}[tbh]
    \centering
    \includegraphics[width=\linewidth]{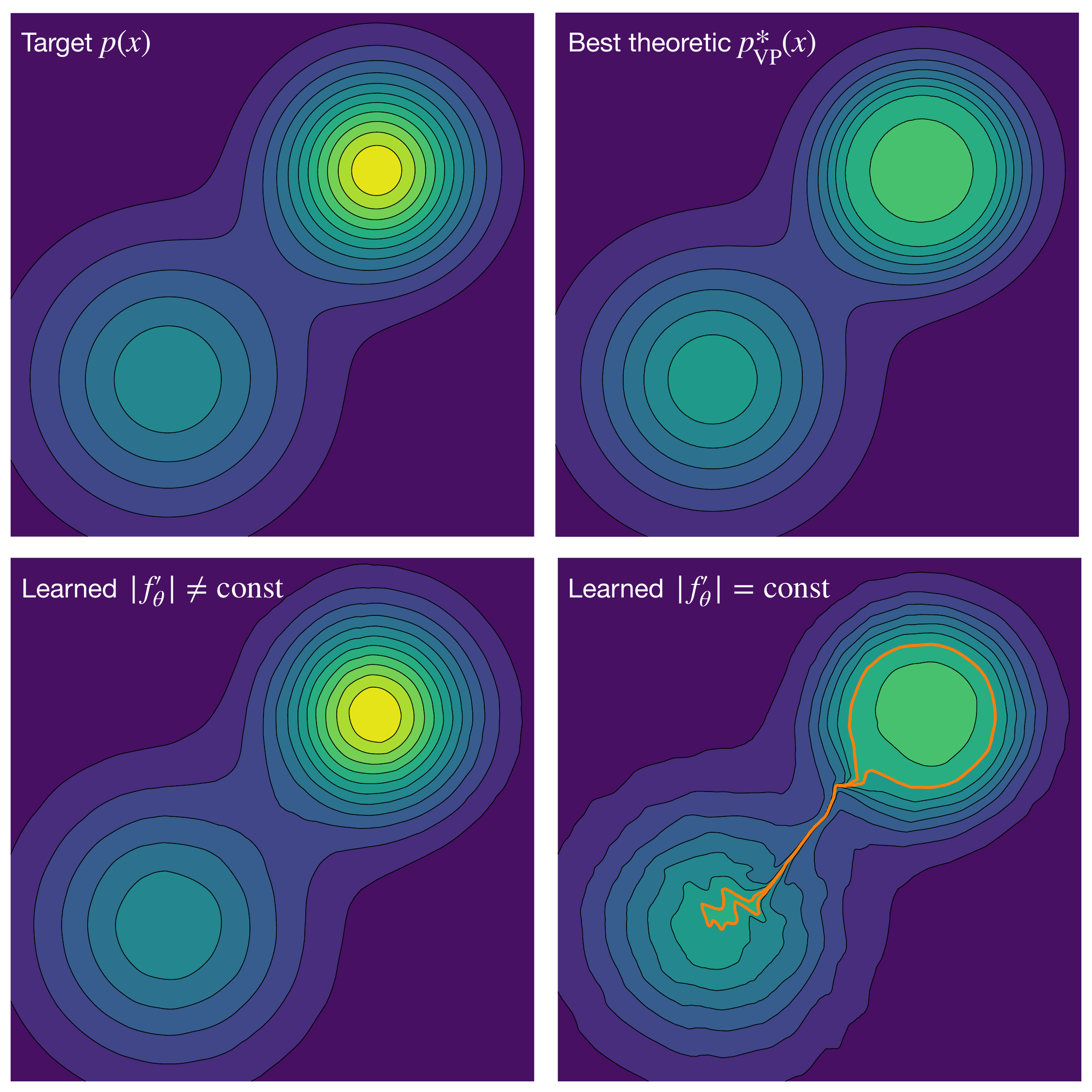}
    \caption{
        \textbf{We reveal two limitations of volume-preserving flows}: First, a 2D bimodal distribution \textit{(top left)} cannot be represented by a volume-preserving flow, the theoretic optimum predicted by \cref{thm:volume-preserving-lower-bound} assigns wrong densities to both modes \textit{(top right)}. Learning a volume-preserving flow comes very close to this suboptimal solution \textit{(bottom right)}. Second, since the flow is continuous in practice, it cannot represent multi-modal distributions by \cref{thm:mode-equivalence}, but a small bridge of high density remains \textit{(orange level set)}. A normalizing flow with variable Jacobian determinant does not have these issues \textit{(bottom left)}.
    }
    \label{fig:effet-const-jac-example}
\end{figure}

The following result formalizes that volume-preserving flows are not universal:
\begin{theorem}
    \label{thm:volume-preserving-not-universal}
    The family of normalizing flows with constant Jacobian determinant $|f_\theta'(x)| = \operatorname{const}$ is \emph{not} a universal distribution approximator under KL divergence.
\end{theorem}
The proof constructs a concrete $p(x)$ and shows that the KL is bounded from below by a finite value for every possible value of $C$ (see \cref{app:proof-volume-preserving-not-universal}).

\Cref{fig:effet-const-jac-example} also shows another shortcoming of volume-preserving flows as they are practically implemented: There is a thin bridge of high density between the two modes. The reason is that flows are implemented as continuous invertible functions (as opposed to \cref{thm:volume-preserving-lower-bound,thm:volume-preserving-not-universal}, which only require invertibility). This makes the modeled distribution $p_\theta(x)$ inherit the mode structure of the latent $p(z)$:
\begin{proposition}
    \label{thm:mode-equivalence}
    A normalizing flow $p_\theta(x)$ based on diffeomorphisms $f_\theta(x)$ with constant Jacobian determinant $|f_\theta'(x)| = \operatorname{const}$ has the same number of modes as the latent distribution $p(z)$.
\end{proposition}
The proof in \cref{app:mode-equivalence} uses that diffeomorphisms map open sets to open sets, and thus the neighborhoods of density maxima in the latent space remain neighborhoods of density maxima in the data space.

The constructions underlying \cref{thm:volume-preserving-lower-bound,thm:volume-preserving-not-universal,thm:mode-equivalence} suggest concrete steps to make volume-preserving flows universal: First, one can write the latent distribution in hyperspherical coordinates $p(z) = p(r) p(\Omega|r)$ and replace the one-dimensional distribution $p(r)$ by a learned variant. This can even be done post-training, as the proof of \cref{thm:volume-preserving-lower-bound} reveals that a volume-preserving flow learns the same transport map whenever the latent distribution is rotationally symmetric and the density decreases with the distance from the origin. Second, a larger number of modes can be created by using (learnable) multi-modal distributions in the latent space. We leave implementing these practical improvements open to future work.

Together, we identify a fundamental limitation for applications based on volume-preserving flows. It explains why RealNVP significantly outperforms NICE in practice \cite{dinh2017density}. Work using volume-preserving flows must take this limited expressivity and the resulting biases in the learned distributions into account. In \cref{sec:shortcomings}, we show that this problem also applies to the most recent universality proof for coupling-based normalizing flows by \citet{koehler2021representational}.

\section{Universality of Coupling-Based Normalizing Flows}
\label{sec:coupling-universality}

In this section, we present our improved universality proof for non-volume-preserving coupling-based normalizing flows. It overcomes limitations of previous constructions and relies on the simple idea of iteratively training coupling blocks.

\subsection{Problems with Existing Constructions}
\label{sec:shortcomings}

The existing proofs \cite{teshima2020couplingbased, koehler2021representational} that affine %
and more expressive coupling flows are distributional universal approximators come with several limitations. In particular, their constructions use ill-conditioned coupling blocks as the translations $t(b)$ approximate step functions, as noted by \citet{koehler2021representational}. Also, they give guarantees only on a compact subspace $K \subset \RR^D$, and \citet{teshima2020couplingbased} use only one active dimension per coupling. This limits their practical applicability.

In addition, the flow constructed in \citet{koehler2021representational} is volume-preserving and thus not universal in KL divergence by our \cref{thm:volume-preserving-not-universal}. This is possible since they show convergence under Wasserstein distance $W_2$ in \citep[Theorem 1]{koehler2021representational}, which is has limited sensitivity to volume-preserving flows.

It is easy to see that their flow is volume-preserving by looking at the scaling functions $s(b)$ in the three affine coupling layers (\cref{eq:affine-coupling}) they use. They read: $s^{(1)} = \epsilon'$ and $s^{(2)} = s^{(3)} = \epsilon''$. This means that the overall flow has a Jacobian determinant $|f'_{\theta}| = (\epsilon' \epsilon''^2)^{\frac{D}{2}}$. This volume change is independent of the input, making the flow volume-preserving. 

Also note how the volume change is directly tied to the guaranteed Wasserstein distance, since they guarantee that $W_2(p_\theta(x), p(x)) < \epsilon$, and the above scalings fulfill $\epsilon'' \ll \epsilon' \ll \epsilon$. Thus, the volume change $|f'_{\theta}| \ll \epsilon^{\frac{3D}{2}}$ vanishes, and its inverse $|(f^{-1})'_{\theta}| \gg \epsilon^{-\frac{3D}{2}}$ explodes as $\epsilon$ is reduced, rendering the flow ill-conditioned regardless of the distribution at hand. This is additional to the ill-condition of the translation terms $t(b)$ approximating step functions. %

Together, this calls for a new universality guarantee that is based on a new coupling flow construction. Our new construction, presented in the following sections, uses a flow that is neither arbitrarily ill-conditioned nor volume-preserving. Also, it converges globally and considers vanilla affine coupling blocks.

\subsection{Convergence Metric}
\label{sec:convergence-metric}

Ideally, a universality theorem gives a guarantee on the Kullback-Leibler divergence (\cref{eq:loss}) to measure the progress of convergence. The KL is the loss used in practice and implies weak convergence (ensuring convergence of expectation values), and convergence of densities \cite{gibbs2002choosing}. Also, as we showed previously in \cref{sec:limitations-vol-preserving-flows}, the KL is able to distinguish the expressivity between volume-preserving and non-volume-preserving flows. %

The metric of convergence we consider in our proof is related to the Kullback-Leibler divergence. To construct it, consider the image of $p(x)$ through the flow $f_\theta(x)$, that is the latent distribution the flow actually creates:
\begin{equation}
	p_\theta(z) = ((f_\theta)_\sharp p(x))(z) = p(f_\theta^{-1}(z))|(f^{-1}_\theta)'(z) |.
\end{equation}
Now rewrite the loss $\Ll$ in \cref{eq:loss} into a form which compares $p_\theta(z)$ to the target latent distribution $p(z)$:
\begin{align}
    \Ll &
    = \KL{p(x)}{p_\theta(x)} 
    \\&
    = \int p(x) \log \frac{p(x)}{p(z=f_\theta(x))|f_\theta'(x)|}dx \\&
    = \int p(z) |f_\theta^{-1}{}'(z)| \log \frac{p(f_\theta^{-1}(z))|f_\theta^{-1}{}'(z)|}{p(z)}dz \\&
    = \KL{p_\theta(z)}{p(z)}. \label{eq:loss-latent-space}
\end{align}
This identity shows that the divergence between the true $p(x)$ and the model $p_\theta(x)$ can equally be measured in the latent space, via the KL divergence between the current latent distribution that the model generates from the data $p_\theta(z)$ and the target latent distribution $p(z)$.

Let us now consider what happens if we append one more affine coupling block $f_{\textnormal{blk}; \theta_+}$ to an existing normalizing flow $f_\theta(x)$, resulting in a flow which we call $p_{\theta \cup \theta_+}(x)$. Let us choose the parameters of the additional coupling block $\theta_+$ such that it maximally reduces the loss without changing the previous parameters:
\begin{equation}
    \label{eq:minimal-layer-loss}
    \min_{\theta_+} \KL{p_{\theta \cup \theta_+}(z)}{p(z)}.
\end{equation}
We now propose to use the difference in loss before and after this layer to measure the convergence of the flow:
\begin{definition}
	\label{def:convergence-measure}
	Given a normalizing flow $f_\theta(x)$ on a continuous probability distribution $p(x)$ with finite first and second moment and $p(x) > 0$ everywhere.
	Then, we define the \emph{loss improvement by an affine coupling block} as:
	\begin{equation}
		\label{eq:loss-improvement}
		\begin{matrix*}[l]
			\Delta_{\textnormal{affine}}(p_\theta(z)) \\
			\displaystyle := \KL{p_{\theta}(z)}{p(z)} - \min_{\theta_+} \KL{p_{\theta \cup \theta_+\!}(z)}{p(z)},
		\end{matrix*}
	\end{equation}
	where $\theta_+ = (Q, \phi)$ parameterizes a single $L$-bi-Lipschitz affine coupling block whose conditioner neural network $\psi_{\phi}$ has at least two hidden layers of finite width and ReLU activations.
\end{definition}
The coupling block is restricted to be bi-Lipschitz in order to be well-conditioned. This means that we can choose $L > 1$ such that $L^{-1} < \|f_\text{cpl}(x) - f_\text{cpl}(y)\|/\|x - y\|  < L$, making both forward and inverse passes through each coupling well-conditioned. Choosing a smaller $L$ will result in more coupling blocks, each more numerically stable. Note that we assume ReLU networks for mathematical convenience, but think that the definition is equivalent to versions with different activation functions.

The following theorem shows that $\Delta_{\textnormal{affine}}(p_\theta(z))$ is a useful measure of convergence to the standard normal distribution:
\begin{theorem}
    \label{thm:gaussian-identification}
    With the definitions from \cref{def:convergence-measure}:
    \begin{equation}
        p_\theta(z) = \Nn(z; 0, I)
        \Leftrightarrow
        \Delta_{\textnormal{affine}}(p_\theta(z)) = 0.
    \end{equation}
\end{theorem}
This shows that the loss improvement by an affine coupling block~$\Delta_{\textnormal{affine}}(p_\theta(z))$ is a useful convergence metric for normalizing flows: If adding more layers has no effect, the latent distribution has converged to the right distribution.

In the remainder of this section, we give a sketch of the proof of \cref{thm:gaussian-identification}, with technical details moved to \cref{proof:gaussian-identification}. We will continue with our universality theorem in \cref{sec:affine-coupling-universality}.

We proceed as follows: First, we use an explicit form of the maximal loss improvement $\Delta_{\textnormal{affine}}^*(p_\theta(z))$ for infinitely expressive affine coupling blocks \cite{draxler2020characterizing}. Then, we show in \cref{lem:convergence-equivalence} that convergence of these unrealistic networks is equivalent to convergence of finite ReLU networks. Finally, we show that $\Delta_{\textnormal{affine}}(p_\theta(z)) = 0$ implies $p_\theta(z) = \Nn(z; 0, I)$. While this derivation is constructed for affine coupling blocks, it also holds for coupling functions which are more expressive (see \cref{app:compatible-couplings} for all applicable couplings we are aware of): If an affine coupling block cannot make an improvement, neither can a more expressive coupling. The other direction is trivial, since by $p_\theta(z) = \Nn(0, I)$, no loss improvement is possible.

If we assume for a moment that neural networks can exactly represent arbitrary continuous functions, then this hypothetical maximal loss improvement was computed by \citet[Theorem 1]{draxler2020characterizing}. A single affine coupling block with a fixed rotation layer $Q$, in order to maximally reduce the loss, will standardize the data by normalizing the first two moments of the active half of dimensions $a = (Qx)_{D/2 + 1, \dots, D}$ conditioned on the passive half of dimensions $b = (Qx)_{1, \dots, D}$. The moments before the coupling
\begin{align}
    \label{eq:conditional-moments}
    \EE_{a_i|b}[a_i] &= m_i(b), &\Var_{a_i|b}[a_i] &= \sigma_i(b)
\intertext{are mapped to:}
    \label{eq:conditionally-normalized}
    \EE_{\tilde a_i|b}[\tilde a_i] &= 0, &\Var_{\tilde a_i|b}[\tilde a_i] &= 1.
\end{align}
This is achieved via the following affine transformation, shifting the conditional mean to zero and scaling the conditional standard deviation to one:
\begin{equation}
    \label{eq:action-single-layer}
    \tilde a_i(a_i; b) = \frac{1}{\sigma_i(b)}(a_i - m_i(b)).
\end{equation}
In terms of loss, this transformation can at most achieve the following loss improvement, with a contribution from each passive coordinate $b$:
\begin{equation}
    \label{eq:best-loss-improvement}
    \begin{matrix*}[l]
	    \Delta_{\textnormal{affine}}^*(p_\theta(z)) \\
	    \displaystyle = \max_Q \tfrac12 \sum_i^{D/2} \EE_{b}\big[ \underbrace{m_i^2(b)}_\text{(A)} + \underbrace{\sigma_i^2(b) - 1 - \log \sigma_i^2(b)}_\text{(B)} \big],
    \end{matrix*}
\end{equation}
where the dependence on $Q$ enters through $(b, a) = Qx$. With the asterisk, we denote that this improvement cannot necessarily be reached in practice with finitely sized and well-conditioned neural networks. More expressive coupling functions can reduce the loss stronger, which we make explicit in \cref{sec:expressive-universality}.

What loss improvement can be achieved if we go back to finite neural networks? It turns out that $\Delta_{\textnormal{affine}}^*(p_\theta(z)) > 0$ is equivalent to the existence of a well-conditioned coupling block as in \cref{def:convergence-measure} with $\Delta_{\textnormal{affine}}(p_\theta(z)) > 0$:
\begin{lemma}
    \label{lem:convergence-equivalence}
    Given a continuous probability density $p(z)$ on $z \in \RR^k$. Then,
    \begin{align}
        \Delta_{\textnormal{affine}}^*(p_\theta(z)) &> 0
    \intertext{if and only if:}
        \label{eq:neural-net-reduces-loss}
        \Delta_{\textnormal{affine}}(p_\theta(z)) &> 0.
    \end{align}
\end{lemma}
This says that the events $\Delta_{\textnormal{affine}}(p_\theta(z)) = 0$ and $\Delta_{\textnormal{affine}}^*(p_\theta(z)) = 0$ can be used interchangeably. The equivalence comes from the fact that if $\Delta_{\textnormal{affine}}^*(p_\theta(z)) > 0$, then we can always construct a conditioner neural network that scales the conditional standard deviations closer to one or the conditional means closer to zero, reducing the loss.
In the detailed proof in \cref{proof:convergence-equivalence} we also make use of a classical regression universal approximation theorem \cite{hornik1991approximation} and ensure the additional coupling block is well-conditioned.

Finally, if the first two \textit{conditional} moments of any latent distribution $p(z)$ are normalized for all rotations $Q$:
\begin{equation}
    \EE_{a_i|b}[a_i] = m_i(b) = 0, \quad \Var_{a_i|b}[a_i] = \sigma_i(b) = 1,
\end{equation}
then the distribution must be the standard normal distribution: $p(z) = \Nn(z; 0, I)$: \Cref{eq:best-loss-improvement} enforces two characteristics of $p_\theta(z)$ that uniquely identify the standard normal distribution: (A)~Is $p_\theta(z)$ rotationally symmetric, since $m_i(b) = 0$ for all $Q$ holds only for rotationally symmetric distributions \cite{eaton1986characterization}. (B)~This term non-negative and zero only for $\sigma_i(b) = 1$ for all $Q$, which uniquely identifies the standard normal from all rotationally symmetric distributions \cite{bryc1995normal}.

This concludes the proof sketch of \cref{thm:gaussian-identification} and we are now ready to present our universality result, employing $\Delta_{\textnormal{affine}}(p_\theta(z))$ as a convergence metric.

\subsection{Affine Coupling Flows Universality}
\label{sec:affine-coupling-universality}

We now confirm that affine coupling flows are a distributional universal approximator in terms of the convergence metric we derived in \cref{sec:convergence-metric}:
\begin{theorem}
    \label{thm:affine-coupling-universality}
    For every continuous $p(x)$ with finite first and second moment with infinite support, there is a sequence of normalizing flows $f_n(x)$ consisting of $n$ $L$-bi-Lipschitz affine coupling blocks such that:
    \begin{align}
    	p_n(z) \xrightarrow{n \to \infty} \Nn(z; 0, I),
    \end{align}
    in the sense that $\Delta_\textnormal{affine}(p_n(z)) \xrightarrow{n \to \infty} 0$.
\end{theorem}
This means that with increasing depth, the latent distribution of the flow converges to the standard normal.

The proof of \cref{thm:affine-coupling-universality} explicitly constructs a normalizing flow by following an iterative scheme. We start with the data distribution as our original guess for the latent distribution: $p_0(z) = p(x=z)$. Then, we repeatedly append individual affine coupling blocks $f_\text{blk}(x)$ consisting of a rotation $Q$ and a coupling $f_\text{cpl}$, optimizing the new parameters to maximally reduce the loss as in \cref{eq:minimal-layer-loss}.

This series of coupling blocks converges: $\Delta_{\textnormal{affine}}(p_\theta(z))$ measures how much adding each affine coupling block reduces the loss, but the total loss that can be reduced by the concatenation of many blocks is bounded. Since improvements $\Delta_{\textnormal{affine}}(p_\theta(z))$ are also non-negative, they must converge to zero for the sum to be finite \citep[Theorems 3.14 and 3.23]{rudin1976principles}. By \cref{thm:gaussian-identification}, the fixed point of this procedure is a standard normal distribution in the latent space. We give the full proof in \cref{app:proof-affine-coupling-universality}.

\Cref{fig:universality-example} shows an example for how \cref{thm:affine-coupling-universality} constructs the coupling flow in order to learn a toy distribution. The affine coupling flow is able to learn the distribution well, despite its difficult topology. Empirically, this is also true in terms of KL divergence: \Cref{fig:layer-improvement-vs-kl} in \cref{app:details-layer-wise-experiment} shows the relation between $\Delta_\textnormal{affine}(p_\theta(z))$ and the KL divergence for the flow, both of which decrease over the course of training.

\Cref{tab:improvements} summarizes how our construction is closer to practice than previous work \cite{teshima2020couplingbased, koehler2021representational}: We only use well-conditioned $L$-bi-Lipschitz couplings and allow variable volume change $|f_\theta'(x)|$, as evidenced by the rescaling term in \cref{eq:action-single-layer}. We give further details on the sensitivity of $\Delta_\text{affine}(p_\theta(z))$ to volume-preserving transformations in \cref{sec:volume-perserving-under-delta}.

\textbf{Limitations:} Despite these advances, there are some properties we hope can be improved in the future: 
First, our construction shows that we can build a deep enough flow with arbitrary precision, but we have not exploited that blocks can be jointly optimized. Thus, while our construction shows universality of end-to-end training, we expect a flow trained this way to require fewer blocks than our iterative proof for the same performance.
Secondly, it is unclear how the convergence metric \cref{sec:convergence-metric} is related to convergence in the loss used in practice, the KL divergence given in \cref{eq:loss}. In practice, we find that constructing a coupling flow through iterative training converges in KL divergence (see \cref{fig:layer-improvement-vs-kl} in \cref{app:details-layer-wise-experiment}), so we conjecture that our way of constructing a universal coupling flow converges in KL divergence. The reverse holds: We show in \cref{corr:kl-implies-improvement} in \cref{app:kl-implies-improvement} that convergence in KL implies convergence under our new metric.
Finally, our proof gives no guarantee on the number of required coupling blocks to achieve a certain performance. Related work shows that the number of layers is constant with dimension for the special case of Gaussian data \cite{koehler2021representational, draxler2022whitening}, but in practice is a hyperparameter that is to be tuned depending on the data and together with the complexity of the subnetworks. %
We hope that our contribution paves the way towards a full understanding of affine coupling-based normalizing flows.

\subsection{Expressive Coupling Flow Universality}
\label{sec:expressive-universality}

The above \cref{thm:affine-coupling-universality} shows that affine couplings $c(a_i; \theta) = s a_i + t$ are sufficient for universal distribution approximation. As mentioned in \cref{sec:coupling-flows}, a plethora of more expressive coupling functions have been suggested, for example neural spline flows \cite{durkan2019neural} that use monotone rational-quadratic splines as the coupling function. It turns out that by choosing the parameters in the right way, all coupling functions we are aware of can exactly represent an affine coupling, except for the volume-preserving variants, see \cref{app:compatible-couplings}. For example, a rational quadratic spline can be parameterized as an affine function by using equidistant knots $(a_k, \tilde a_k)$ such that $\tilde a_k = s a_k + t$ and fixing the derivative at each knot to $s$.

\begin{figure}
    \centering
    \includegraphics[width=\linewidth]{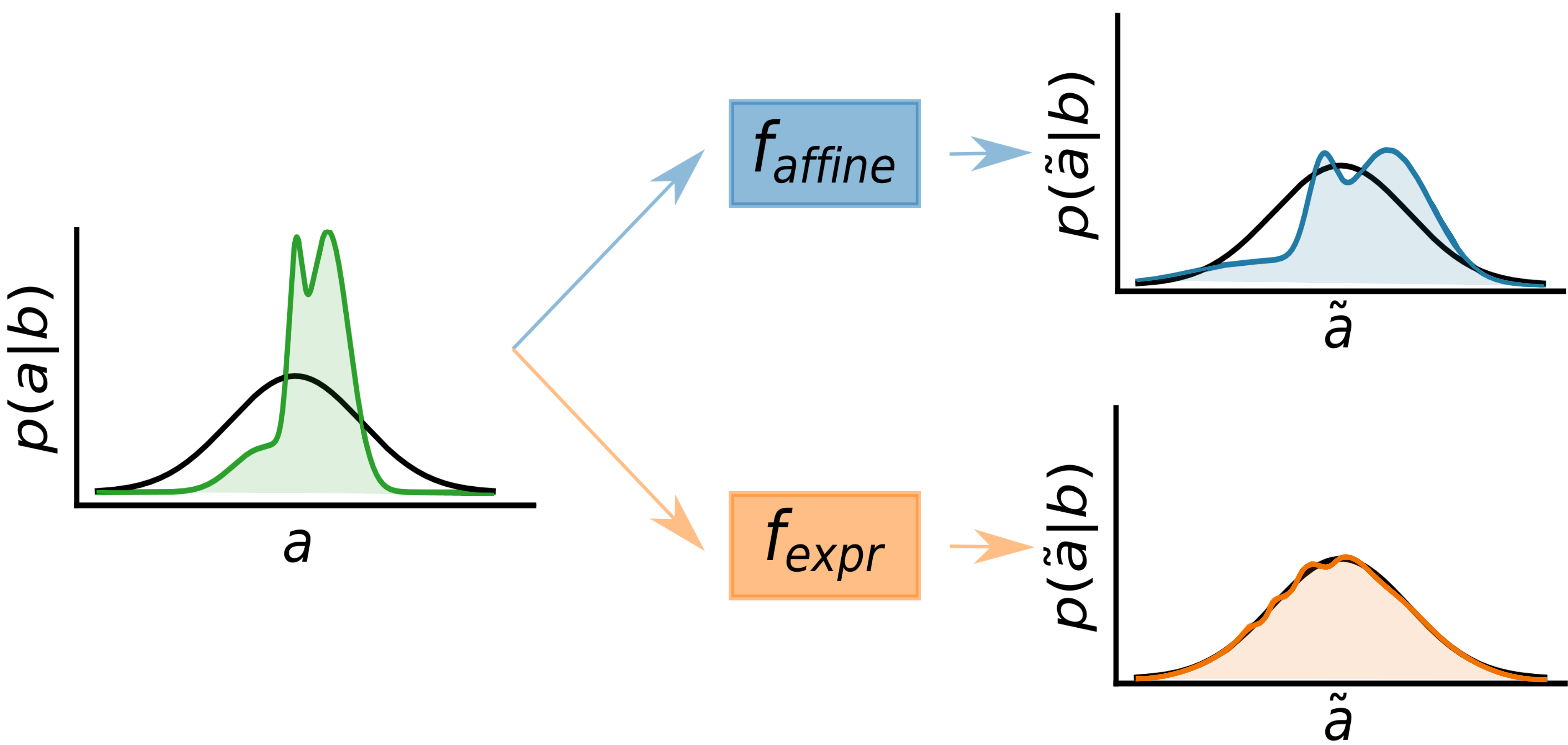}
    \caption{
        Only coupling functions strictly more expressive than affine can fit non-Gaussian conditionals $p(a_i|b)$ in a single block, resulting in a faster loss decrease (\cref{eq:improvement-split}). Note that coupling flows of both kinds are universal.
    }
    \label{fig:expressive-coupling}
\end{figure}

Thus, the universality of more expressive coupling functions follows immediately from \cref{thm:affine-coupling-universality}, just like \citet{ishikawa2022universal} extended their results from affine to more expressive couplings:
\begin{corollary}
    \label{thm:expressive-coupling-universality}
	For every continuous $p(x)$ with finite first and second moment with infinite support, there is a sequence of normalizing flows $f_n(x)$ consisting of $n$ coupling blocks \emph{with coupling functions at least as expressive as affine couplings} such that:
	\begin{align}
		p_n(z) \xrightarrow{n \to \infty} \Nn(z; 0, I),
	\end{align}
	in the sense that $\Delta_\textnormal{affine}(p_n(z)) \xrightarrow{n \to \infty} 0$.
\end{corollary}

Our proof of \cref{thm:affine-coupling-universality}, constructed through layer-wise training, shows how more expressive coupling functions can outperform affine functions using the same number of blocks.
Similar to the loss improvement for an affine coupling in \cref{eq:loss-improvement}, let us compute the maximally possible loss improvement for an arbitrarily flexible coupling function:
\begin{equation}
    \label{eq:improvement-split}
    \Delta_{\textnormal{universal}}^* = \max_Q \EE_{b}[J(b) + \Delta_{\textnormal{affine}}^*(Q)] \geq \Delta_{\textnormal{affine}}^*,
\end{equation}
where the expectation again goes over the passive coordinate $b = (Qx)_{1, \dots, D/2}$ and we have omitted the dependency on $p_\theta(z)$ for brevity.

Here, the additional loss improvement is the conditional negentropy $J(b) = \sum_{i=1}^{D/2} \KL{p_\theta(a_i|b)}{\Nn(m_i(b), \sigma_i(b)}$, which measures the deviation of each active dimension from a Gaussian distribution with matching mean and variance.
An affine coupling function $c(a_i; \theta) = sa_i + t$ doesn't influence this term, due to its symmetrical effect on both sides of the KL in $J(p)$ \citep[Lemma 1]{draxler2022whitening}. More expressive coupling blocks, however, are able to tap on this loss component if the conditional distributions $p(a_i|b)$ are significantly non-Gaussian, see \cref{fig:expressive-coupling} for an example.

The impact of this gain likely varies with the dataset. For instance, in images, the distribution of one color channel of one pixel conditioned on the other color channels in the entire image, often shows a simple unimodal pattern with low negentropy. This may explain why affine coupling blocks are enough to learn the distribution of images \cite{kingma2018glow}.
We give additional technical details on \cref{eq:improvement-split} and the subsequent arguments in \cref{app:pythagorean-split}.

\section{Conclusion}

Our new universality proofs show an intriguing hierarchy of the universality of different coupling blocks:
\begin{enumerate}[itemsep=1pt,topsep=0pt]
    \item \textit{Volume-preserving normalizing flows} are not universal in KL divergence and thus cannot learn all targets $p(x)$.
    \item \textit{Affine coupling flows} such as RealNVP \cite{dinh2017density} are distributional universal approximators despite their seemingly restrictive architecture.
    \item Coupling flows with \textit{more expressive coupling functions} are also universal approximators, but they converge faster by tapping on an additional loss component in layer-wise training.
\end{enumerate}

Our work theoretically grounds why coupling blocks are the standard choice for practical applications with normalizing flows, combined with their easy implementation and speed in training and inference.
We remove spurious constructions present in previous proofs and use a simple principle instead: Construct a flow layer by layer until no more loss improvement can be achieved.

Using volume-preserving flows may have negatively affected existing work. We show what distribution $p^*(x)$ they approximate instead of the true target $p(x)$ and propose how universality can be recovered by learning the actual latent distribution after training.

\section*{Impact Statement}

This paper presents work whose goal is to advance the field of Machine Learning. There are many potential societal consequences of our work, none of which we feel must be specifically highlighted here.

\section*{Acknowledgements}

This work is supported by Deutsche Forschungsgemeinschaft (DFG, German Research Foundation) under Germany's Excellence Strategy EXC-2181/1 - 390900948 (the Heidelberg STRUCTURES Cluster of Excellence).
It is also supported by the Vector Stiftung in the project TRINN (P2019-0092).
We thank Armand Rousselot and Peter Sorrenson for the fruitful discussions and feedback. We thank Vincent Souveton for the helpful comments regarding Hamiltonian Monte Carlo.

\bibliography{references}
\bibliographystyle{icml2024}

\newpage
\appendix
\onecolumn

\section{Architectures}
\subsection{Compatible Coupling Functions}
\label{app:compatible-couplings}

The following lists all coupling functions $c(a; \psi)$ (see \cref{eq:coupling-layer} for its usage) we are aware of. Our universality guarantees \cref{thm:affine-coupling-universality,thm:expressive-coupling-universality} hold for all of them:
\begin{itemize}
    \item \textbf{Affine coupling flows} as RealNVP \cite{dinh2017density} and GLOW \cite{kingma2018glow}:
    \eql{
        \label{eq:affine-coupling-again}
        c(x; \psi) = s x + t.
    }
    Here, $\psi = [s; t] \in \RR_+ \times \RR$. Note that NICE ($s=1$, \citet{dinh2015nice}) and GIN ($\prod_{i=1}^{D/2} s_i = 1$, \cite{sorrenson2019disentanglement}) follow the same functional form, but are volume-preserving and thus not universal (see \cref{sec:limitations-vol-preserving-flows}).
    
    \item \textbf{Nonlinear squared flow} \cite{ziegler2019latent}:
    \eql{
        c(x; \psi) = ax + b + \frac{c}{1 + (dx + h)^2},
    }
    for $\psi = [a, b, c, d, h] \in \RR_+ \times \RR^4$.
    Choose $c = 0$ to obtain an affine coupling.

    \item \textbf{Flow++} \cite{ho2019flow}:
    \eql{
        c(x; \psi) = s \sigma^{-1}\left( \sum_{j=1}^K \pi_j \sigma\left(\frac{x - \mu_j}{\sigma_j}\right) \right) + t.
    }
    Here, $\psi = [s; t; (\pi_j, \mu_j, \sigma_j)_{j=1}^K] \in \RR_+ \times \RR \times (\RR \times \RR \times \RR_+)^K$ and $\sigma$ is the logistic function. Choose all $\pi_j = 0$ except for $\pi_1 = 1$, all $\mu_j = 0$ and all $\sigma_j = 1$ to obtain an affine coupling.

    \item \textbf{SOS polynomial flows} \cite{jaini2019sumofsquares}:
    \eql{
        c(x; \psi) = \int_0^x \sum_{\kappa=1}^k \left(\sum_{l=0}^r a_{l,\kappa} u^l \right)^2 \d u + t.
    }
    Here, $\psi = [t; (a_{l, \kappa})_{l, \kappa}] \in \RR \times \RR^{r k}$. Choose all $a_{l,\kappa} = 0$ except for $a_{1, 0} = s$ to obtain an affine coupling.
    
    \item \textbf{Spline flows} in all variants:  Cubic \cite{durkan2019cubicspline}, piecewise-linear, monotone quadratic \cite{muller2019neural}, and rational quadratic \cite{durkan2019neural} splines. A spline is parameterized by knots $\theta$ with optional derivative information depending on the spline type, and $c$ computes the corresponding spline function. Choose the spline knots as $y_i = s x_i + t$ for an affine coupling, choose the derivatives as $x_i' = s$ for an affine coupling.
    
    \item \textbf{Neural autoregressive flow} \cite{huang2018neural} use a feed-forward neural network to parameterize $c(x; \theta)$ by a feed-forward neural network. They show that a neural network is guaranteed to be bijective if all activation functions are strictly monotone and all weights positive. One can construct a ReLU network with a single linear region to obtain an affine coupling.
    
    \item \textbf{Unconstrained monotonic neural networks} \cite{wehenkel2019unconstrained} also use a feed-forward neural, but restrict it to have positive output. To obtain $c(x; \theta)$, this function is then numerically integrated with a learnable offset for $x=0$. Choose a constant neural network to obtain an affine coupling.
\end{itemize}

\subsection{Volume-Preserving Normalizing Flows}
\label{app:volume-preserving-architectures}

Below, we list the ways to construct volume-preserving flows we are aware of. Our non-universality results \cref{thm:volume-preserving-lower-bound,thm:volume-preserving-not-universal,thm:mode-equivalence} hold for all of them:
\begin{itemize}
    \item \textbf{Nonlinear independent components estimation} (NICE) \cite{dinh2015nice} is a coupling block:
    \eql{
        c(x; \psi) = x + t.
    }
    Here, $\psi = t \in \RR$. 
    \item \textbf{General Incompressible-flow Networks} (GIN) \cite{sorrenson2019disentanglement} generalize NICE by allowing the individual dimensions to change volume, only the overall volume change is normalized:
    \eql{
        c_i(x_i; \psi) = \frac{s_i}{\prod_{j=1}^{D/2} s_j} x_i + t_i,
    }
    Here, $\psi = [s, t] \in \RR_+^{D/2} \times \RR^{D/2}$ is jointly predicted for all active dimensions and then normalized as above.
    \item \textbf{Neural Hamiltonian Flows} \cite{toth2020hamiltonian} parameterize a Neural ODE as a Hamiltonian system:
    \begin{equation}
        \frac{d q}{d t} = \frac{\partial \Hh}{\partial p}, \qquad \frac{d p}{d t} = -\frac{\partial \Hh}{\partial q}
    \end{equation}
    The Hamiltonian $\Hh(p, q)$ is a real-valued function that is parameterized by a neural network. Its derivatives are obtained via automatic differentiation. The variant \textbf{Fixed-kinetic Neural Hamiltonian Flows} \cite{souveton2024fixedkinetic} fixes the kinetic term of the Hamiltonian to $K(p)=\frac12 p^T M^{-1} p$, where the positive definite matrix $M$ is learned, and learns the potential $V(q)$ via a neural network to obtain $\Hh(p, q) = K(p) + V(q)$. The solution to the above ODE is volume-preserving on $x = (p, q)$.
\end{itemize}

Note that some works employing volume-preserving flows such as \citet{dibak2022temperature,souveton2024fixedkinetic} consider \textit{augmented flows} \cite{huang2020augmented}, where additional noise dimensions are padded $p(a|x)$ to the data distribution of interest $p(x)$. Then, the flow learns the joint distribution $p(x, a) = p(x) p(a|x)$. Depending on how $p(a|x)$ is constructed, this can positively or negatively impact the expressivity of the considered volume-preserving flow. For example, if $a \perp x$, then the joint distribution $p(x, a)$ has at least the same number of modes as $p(x)$, but the learned joint distribution $p_\theta(x, a)$ can only have a single mode by \cref{thm:mode-equivalence}, inducing a bias. To derive the universality in terms of KL, apply \cref{thm:volume-preserving-lower-bound} to the joint distribution at hand. On the positive side, \citet{souveton2024fixedkinetic} find that having $p(a|x) = \Nn(\mu(x), \sigma(x)^2I)$ brings the obtained $p_\theta(x) = \int p_\theta(x, a) da$ closer to the target. This can be seen having an independent augmentation $a \sim \Nn(0, I) \perp x$ plus a single RealNVP coupling shifting and scaling the augmented dimensions. This effectively breaks the volume-preservation of the flow in the joint space. It is unclear, however, whether this removes all biases from the volume-preserving flow.

\section{Proofs on Volume-Preserving Normalizing Flows}
\label{app:volume-preserving}

\subsection{Minimizer of Volume-Preserving Normalizing Flows}
\label{app:volume-preserving-minimizer}

Volume-preserving flows are not universal in KL divergence, see \cref{thm:volume-preserving-not-universal}. In this section, we consider what distribution a volume-preserving flow converges to instead. We first construct the latent distribution a sufficiently rich volume-preserving flows converges to when trained with KL divergence (\cref{sec:rotationally-symmetric-volume-preserving}) and then show that this actually minimizes the KL (\cref{sec:volume-perserving-under-kl}).

We also demonstrate in what sense $\Delta_\text{affine}$ is sensitive to volume-preserving transformations. We therefore show in \cref{sec:volume-perserving-under-delta} that this flow converges under our convergence measure $\Delta_\textnormal{affine}$ if and only if that volume-preserving flow converges to a standard normal in the latent space under KL divergence.

\subsubsection{Rotationally Symmetric Distribution with Same Level Set Structure}
\label{sec:rotationally-symmetric-volume-preserving}

Let us repeat the change-of-variables formula for a volume-preserving flow from \cref{eq:incompressible-change-of-variables}:
\begin{equation}
    p_\theta(x) = p(z = f_\theta(x))\cdot C,
\end{equation}
where $C = |f_\theta'(x)|$ is constant with respect to $x$. Intuitively, this means that such a flow can \textit{permute} the probability mass at all locations $x$, and apply a single global factor to scale all probability values by spreading out the distribution.

We now construct the best possible distribution learned by a volume-preserving flow to an arbitrary input $p(x)$ in terms of KL divergence. We therefore split the input distribution $p(x)$ into its level sets:
\begin{equation}
    L_v(p(\cdot)) := \{ x \in \RR^D : p(x) = v \}.
\end{equation}
Acting on the input distribution, a volume-preserving flow yields a latent distribution $p_\theta(z)$ whose level set structure is closely related to that of the input distribution, in the following sense:
\begin{equation}
    \label{eq:level-set-volume-change}
    |L_v(p(x))| = \frac{|L_{v/C}(p_\theta(z))|}{C}.
\end{equation}
Intuitively, this captures that a volume-preserving flow maps the input space to the latent space such that the level set in data space for the density value $v$ is mapped to the level set in the latent space of level $v/C$ -- and the $(D-1)$-dimensional volumes are scaled by the factor $C$. Here, we have assumed that the level sets are $(D-1)$-dimensional.

In the following, we use these level sets to construct the distribution $p^*(z)$ a volume-preserving flow converges to in the latent space. This allows us to specify the best solution a volume-preserving flow can converge to. We first consider a fixed $C$ and then later solve the optimization over $C$ separately.
\begin{lemma}
    \label{lem:vp-latent-distribution}
    Let $p(x)$ be a bounded continuous probability density with $(D-1)$-dimensional level sets almost everywhere. Then, a unique continuous probability density $p^*(z)$ with the following properties exists:
    \begin{enumerate}
        \item Its level sets have equal volume: $|L_v(p)| = |L_v(p^*)|$,
        \item $p^*$ is rotationally symmetric: $p^*(z) = p^*(Qz)$ for all $Q \in SO(D)$,
        \item $p^*(z_1, 0, \dots, 0)$ is strictly monotonically decreasing in $0 \leq z_0 < \infty$.
    \end{enumerate}
\end{lemma}
\begin{proof}
    We write $p^*(x) = p^*(r) p^*(\Omega|r)$, where $(r, \Omega)$ are hyper-spherical coordinates. Since $p^*(x)$ should be rotationally symmetric, the distribution of the solid angle $\Omega$ is isotropic, and equal to one over the surface $A_{D-1}(r)$ of the $(D-1)$-dimensional hypersphere:
    \begin{equation}
        p^*(\Omega|r) = \frac{1}{A_{D-1}(r)} = \frac{\Gamma(D/2)}{2 \pi^{\frac{D}{2}} r^{D-1}}
    \end{equation}
    This makes $p^*$ rotationally symmetric and leaves us with constructing $p^*(r)$.

    Define the superlevel sets $v$ for $p(x)$ as follows:
    \begin{align}
        L_v^+(p) = \{ x \in \RR^D : p(x) \geq v \}.
    \end{align}
    Their volume $|L_v^+(p)|$ as measured in $(D-1)$ dimensions is monotonically decreasing in $v$ and its derivative yields the volume of the level set:
    \begin{align}
        \frac{\partial}{\partial v} |L_v^+(p)| = |L_v(p)|.
    \end{align}

    We now demand that
    \begin{equation}
        |L_v^+(p)| = |L_v^+(p^*)|
        \label{eq:equal-super-level-set-volume}
    \end{equation}
    and integrate out $\Omega$:
    \begin{align}
        |L_v^+(p^*)| &= \int \mathbf{1}[p^*(x) \geq v] dx \\
        &= \int d\Omega \int r^{D-1} \mathbf{1} [p^*(r) \geq v]dr\\
        &= A_{D-1}(1) \int r^{D-1} \mathbf{1}[p^*(r) \geq v]dr.
    \intertext{Since $p^*(r)$ should decrease monotonically with $r$, we can replace the indicator function by integral boundaries, where the upper limit depends on the target density value $v$. We identify $\max_x p(x)$ with $R=0$:}
        |L_v^+(p^*)| &= A_{D-1}(1) \int_0^{R(v)} r^{D-1} dr \\
        &= A_{D-1}(1) \frac1D R(v)^{D}
    \end{align}
    Rearranging yields $R(v)$ from $|L_v^+(p)|$:
    \begin{equation}
        R(v) = \left(\frac{D |L_v^+(p)|}{A_{D-1}(1)}\right)^{\frac1D}
    \end{equation}
    As $R(v)$ is monotonous in $|L_v^+(p)|$, $|L_v^+(p)|$ is continuous and monotonous in $v$, and $R(v)$ is invertible, its inverse can be used to define $p^*(r)$:
    \begin{equation}
        p^*(r) = R^{-1}(r).
    \end{equation}

    By choosing $R(v)$ to fulfill \cref{eq:equal-super-level-set-volume}, their derivatives also match:
    \begin{equation}
         |L_v(p)| = \frac{\partial}{\partial v} |L_v^+(p)| = \frac{\partial}{\partial v} |L_v^+(p^*)| = |L_v^+(p^*)|.
    \end{equation}
    The other properties of $p^*$ follow directly from the construction above. The density is unique (up to zero sets) since a rotationally symmetric distribution is uniquely defined by a one-dimensional ray \cite{eaton1986characterization}.
\end{proof}

We now show that the latent distribution $p^*(z)$ is actually attainable by a volume-preserving flow:
\begin{lemma}
    \label{lem:volume-preserving-to-radially-symmetric}
    Under the assumptions of \cref{lem:vp-latent-distribution}. There exists a function $f_\theta: \RR^D \to \RR^D$ that is bijective, continuous and volume-preserving with unit volume change ($|f_\theta'(x)| = 1$) almost everywhere and $p_\theta(z) = p^*(x=z)$.
\end{lemma}
This means that there is a volume-preserving flow that exactly pushes $p(x)$ to its respective $p^*(x)$. Note that this volume-preserving flow has $|f_\theta'(x)| = 1$.
\begin{proof}
    First, note that $p_\theta(z) = p^*(x=z)$ as constructed above can be achieved by a volume-preserving bijection $f_\theta$. To see this, divide the space into the level sets $L_p(v)$ of $p(x)$.
    Since $p(x)$ is continuous, there is a countable sequence of thresholds $v_i$ at which the number of connected components in the level set $L_p(v)$ jumps: The first jump is at $v_{\max} = \max_x p(x)$, below which find as many connected components as there are maximal modes in $p(x)$. The number of connected components changes whenever there is a saddle point or maximum in $p(x)$. Between each subsequent pair of jumps $(v_i, v_{i+1})$, each connected component can be continuously assigned to a countable cluster number. This yields two tessellations of the entire space: One into level sets, and one into continuous connected components. To construct $f_\theta$, we assign the highest points of $x^* \in \RR^D$ to $f_\theta(x^*) = 0$. Then, we continuously arrange the finite number of components until the next jump in $(\max_x p(x), v_1)$ around the origin such that the resulting level sets are concentric circles. This constructs $f_\theta$ that pushes $p(x)$ to $p^*(x)$ and fulfills the above constraints.
\end{proof}

\subsubsection{Best Volume-Preserving Normalizing Flow under KL Divergence}
\label{sec:volume-perserving-under-kl}

We are going to make use of the following identity that a volume-preserving flow with $|f_\theta'(x)| \neq 1$ with a standard normal in the latent space can be written as a volume-preserving flow with $|\tilde f_\theta'(x)| = 1$ and an alternative latent distribution $\tilde p(z) = \Nn(0, C^{-2}I)$:
\begin{lemma}
    \label{lem:absorb-volume-change}
    Given a volume-preserving bijection $f_\theta$ that is diffeomorphic almost everywhere, and with $|f_\theta'(x)| = C$ for some $C > 0$. Then, there exists a volume-preserving bijection $\tilde f_\theta$ with $|\tilde f_\theta'(x)| = 1$ such that:
    \begin{equation}
        ((f_\theta^{-1})_\sharp \Nn(0, I))(x) = ((\tilde f_\theta^{-1})_\sharp \Nn(0, C^{-2} I))(x)
    \end{equation}
\end{lemma}
In other words, the global volume change of a volume-preserving flow can be absorbed into a single scaling last layer at the latent end of the flow.
\begin{proof}
    Let $\tilde f_\theta(x) = C^{-1} f_\theta(x)$, which has $|\tilde f_\theta'(x)|=C/C=1$ and write $\Nn(0, C^{-2} I)$ as the push-forward through $z \mapsto C^{-1} z$.
\end{proof}

Now we show our main \cref{thm:volume-preserving-lower-bound} on volume-preserving flows:
\begin{reptheorem}{thm:volume-preserving-lower-bound}
    Given a continuous bounded input density $p(x)$. Then, for any volume-preserving flow $p_\theta(x)$ with a standard normal latent distribution, the achievable KL divergence is bounded from below:
    \begin{equation}
        \label{eq:vol-preserving-kl}
        \KL{p(x)}{p_\theta(x)} \geq \KL{p^*(z)}{\Nn(0, |\Sigma_{p^*(z)}|^{\frac{1}{D}} I)},
    \end{equation}
    where $p^*(z)$ is constructed as in \cref{lem:vp-latent-distribution}, and $\Sigma_{p^*(z)}$ is its covariance matrix. The minimal loss is achieved for $C = |\Sigma_{p^*(z)}|^{-\frac{1}{2D}}$.

\end{reptheorem}
\begin{proof}
    By \cref{eq:loss-latent-space}, $\KL{p(x)}{p_\theta(x)} = \KL{p_\theta(z)}{\Nn(0, I)}$. The variant in the latent space can be rewritten using the entropy of $p_\theta(z)$ as:
    \begin{equation}
        \label{eq:vol-preserving-kl-split}
        \KL{p_\theta(z)}{\Nn(0, I)} = \KL{p_\theta(z)}{\Nn(0, I)} = -H[p_\theta(x)] - \int p_\theta(z) \log \Nn(z; 0, I) dz.
    \end{equation}

    The entropy of the latent distribution of a volume-preserving flow only depends on the volume change constant $C$, but not on the exact choice of $f_\theta$:
    \begin{align}
        H[p_\theta(z)] &
        = -\int p_\theta(z) \log p_\theta(z) dz \\&
        = -\int p(x) \log(p_\theta(z=f_\theta(x)) C) dx - \log C \\&
        = -\int p(x) \log p(x) dx - \log C \\&
        = H[p(x)] - \log C.
    \end{align}
    Inserting into \cref{eq:vol-preserving-kl-split}:
    \begin{equation}
        \label{eq:vol-preserving-kl-split-latent}
        \KL{p_\theta(z)}{\Nn(0, I)} = \KL{p_\theta(z)}{\Nn(0, I)} = -H[p(x)] + \log C - \int p_\theta(z) \log \Nn(z; 0, I) dz.
    \end{equation}

    Using \cref{lem:absorb-volume-change}, rewrite the last term in \cref{eq:vol-preserving-kl-split-latent} as an integral over $x$:
    \begin{equation}
        -\int p_\theta(z) \log \Nn(z; 0, C^{-2} I) dz
        = -\int p(x) \log \Nn(\tilde f_\theta(x); 0, C^{-2} I) dx.
    \end{equation}
    This reveals the KL is minimized by assigning the highest values of $p(x)$ to the highest values of $\Nn(0, C^{-2} I)$. Since the order of densities $\Nn(z; 0, C^{-2} I)$ is the same regardless of $C$, the assignment $\tilde f_\theta(x)$ does not depend on $C$, which can be estimated separately via
    \begin{equation}
        \label{eq:vp-intermediate-latent}
        \KL{p^*(z)}{\Nn(0, C^{-2}I)}.
    \end{equation}

    Adapting \citep[Proposition 1]{draxler2022whitening}, the KL divergence \cref{eq:vp-intermediate-latent} can be decomposed as follows:
    \begin{equation}
        \KL{p^*(z)}{\Nn(0, C^{-2}I)} = \KL{p^*(z)}{\Nn(0, |\Sigma_z|^{\frac{1}{D}} I)} + \KL{\Nn(0, |\Sigma_z| I)}{\Nn(0, C^{-2} I)},
    \end{equation}
    where $|\Sigma_z|$ is the determinant of the covariance matrix of the latent codes of $p^*$: $\Sigma_z = \operatorname{Cov}_{z \sim p^*(z)}[z]$. The first term is invariant under scaling the latent codes, and the second term is minimal for $C^{-2} = |\Sigma_z|^{\frac{1}{D}}$.

\end{proof}

Note how the remaining loss $\KL{p^*(z)}{\Nn(0, |\Sigma_{p^*(z)}|^{\frac{1}{D}} I)}$ can be reduced to zero by switching to spherical coordinates and learning $p^*(r)$ via a non-volume-preserving one-dimensional distribution.

\subsubsection{Sensitivity of $\Delta_\textnormal{affine}$ to Volume-Preserving Flows}
\label{sec:volume-perserving-under-delta}

We now confirm that if a volume-preserving flow is not universal under KL divergence, it is also not universal under $\Delta_\textnormal{affine}$.

\begin{lemma}
    \label{thm:volume-preserving-under-delta}
    For a family of normalizing flows with constant Jacobian determinant $|f_\theta'(x)| = \operatorname{const}$, such that $\KL{p_\theta(z)}{p^*(z)} \to 0$, it holds that $\Delta_\textnormal{affine} \to 0$ if and only if $\KL{p^*(z)}{\Nn(0, I)} = 0$.
\end{lemma}
\begin{proof}
    By \cref{lem:convergence-equivalence}, we can use $\Delta_\textnormal{affine}^* > 0$ and $\Delta_\textnormal{affine} > 0$ interchangeably.
    By its definition in \cref{eq:best-loss-improvement},
    \begin{equation}
        \Delta_{\textnormal{affine}}^* = \max_Q \tfrac12 \sum_i^{D/2} \EE_{b}\left[ m_i^2(b) + \sigma_i^2(b) - 1 - \log \sigma_i^2(b) \right]
    \end{equation}
    According to \cref{thm:volume-preserving-lower-bound}, the latent distribution $p_\theta(z) = p^*(z=x)$ minimizes the KL in the latent space: $\KL{p_\theta(z)}{\N(0, I)}$ (the exact assignment between data and latent codes is not unique, but all lead to the same latent estimate).

    At this minimum, $p_\theta(z) = p^*(x=z)$. Since $p^*(x)$ is symmetric under rotations, it holds that, it holds that $m_i(b) = 0$ for $(a, b) = Qx$ in all rotations $Q$. However, since $\Var_{a \sim (Q_\sharp p^*)(a)[a_i|b]} \neq 1$ for all $Q$, it holds that
    \begin{equation}
       \Delta_{\textnormal{affine}}^* = \tfrac12 \sum_i^{D/2} \EE_{b}\left[ \sigma_i^2(b) - 1 - \log \sigma_i^2(b) \right],
    \end{equation}
    which evaluates to the same value regardless of $Q$ since the distribution is rotationally symmetric.

    While this minimum may not be exactly achieved by a continuous volume-preserving flow, a sufficiently rich architecture is able to achieve universality $\KL{p_\theta(z)}{\N(0, I)} \to \KL{p^*(z)}{\N(0, I)}$.
    Since the KL divergence implies the convergence of expectation values \cite{gibbs2002choosing}, it holds that $\Delta_\textnormal{affine}^* \to \tfrac12 \sum_i^{D/2} \EE_{b}\left[ \sigma_i^2(b) - 1 - \log \sigma_i^2(b) \right] \geq 0$. Equality holds if and only if $\KL{p^*(z)}{\Nn(0, I)} = 0$.
\end{proof}

Note that the same argument also applies to Wasserstein distance and weak convergence, so this does not indicate that $\Delta_\textnormal{affine}$ is more informative about convergence than these convergence metrics.

\subsection{Proof of \cref{thm:volume-preserving-not-universal}}
\label{app:proof-volume-preserving-not-universal}

\begin{reptheorem}{thm:volume-preserving-not-universal}
    The family of normalizing flows with constant Jacobian determinant $|f_\theta'(x)| = \operatorname{const}$ is \emph{not} a universal distribution approximator under KL divergence.
\end{reptheorem}
\begin{proof}
In this section, we want to present a two-dimensional example, for which no normalizing flow with constant Jacobian determinant can be constructed such that the KL-divergence between the data distribution and the distribution defined by the normalizing flow is zero.

\begin{equation}\label{eq:target-distribution-counter-example}
    p(x,y) = \begin{cases}
        0.9 &\text{if}\;\;\;(x,y) \in \left[-0.5,0.5\right]\times\left[-0.5,0.5\right]\\
        0.9 - k \cdot\left(\left|x\right|  - 0.5\right)&\text{if}\;\;\; \left|x\right|\in \left[0.5,\frac{0.9}{k}+0.5\right] \land \left|y\right|\in \left[0,\left|x\right|\right]\\
        0.9 - k \cdot\left(\left|y\right|  - 0.5\right)&\text{if}\;\;\; \left|y\right|\in \left[0.5,\frac{0.9}{k}+0.5\right] \land \left|x\right|\in \left[0,\left|y\right|\right]\\
        0&\text{else}
    \end{cases}
\end{equation}

The data distribution $p(x,y)$ which has to be approximated by the model is defined in \cref{eq:target-distribution-counter-example}. This data distribution has a constant value of $0.9$ in a box centered around the origin with a side length of one. This region of constant density is skirted by a margin where the density decreases linearly to zero. Outside the decreasing region, the density is zero. The linear decline is governed by the constant $k$ in \cref{eq:target-distribution-counter-example} which has to be chosen such that the density integrates to one. Since our example only requires the region of constant density but not the decaying tails of it, the exact functional form of the decaying regions are not relevant as long as they lead to a properly normalized distribution. \cref{eq:target-distribution-counter-example} only provides a possible definition of such a density.

To approximate this data distribution, a normalizing flow as defined in \cref{sec:coupling-flows} is considered. In this example, we focus on normalizing flows with constant Jacobian determinant. To simplify notation, we define $J = \left|f_{\theta}'(x)\right| = \operatorname{const}$.
\begin{align}
    A &= \left\{(x,y)\in\RR^2: 0.9 - \epsilon < p_{\theta}(x,y)\right\}\label{eq:set-A-counter-example}\\
    B &= \left[-0.5,0.5\right]\times\left[-0.5,0.5\right] \label{eq:set-B-counter-example}\\
    \bar A &= B \char`\\ A \label{eq:A-complement-on-B-counter-example}
\end{align}
We choose $\epsilon = 0.1$ and use this constant to define the set $A$ (see  \cref{eq:set-A-counter-example}). In addition we define $B$ which is the region of the data space, where the data distribution has a constant value of $0.9$ (see \cref{eq:set-B-counter-example}). $\bar A$ is the complement of $A$ on B (see \cref{eq:A-complement-on-B-counter-example}).

The aim of this example, is to find lower bounds for the KL-divergence between the data distribution and the distribution defined by the normalizing flow. To find these bounds we use Pinsker's inequality \cite{gibbs2002choosing} which links the total variation distance to the Kullback-Leibler divergence:
\begin{equation}\label{eq:pinskers-inequality}
    \delta(p,p_{\theta})  = \sup_{A\; \text{measurable}} \left|P(A) - P_{\theta}(A)\right| \leq \sqrt{\frac{1}{2}D_{KL}(p||p_\theta)}
\end{equation}
It is worth mentioning that constructing one measurable event for which $\left|P(A) - P_{\theta}(A)\right| > 0$ provides a lower bound for the total variation distance and therefore for the KL divergence.

To construct such an event, we consider two distinct cases, which consider different choice for the normalizing flow, characterized by the value of the absolute Jacobian determinant.

\textbf{Case 1: $A = \emptyset$}

This case arises if the absolute Jacobian determinant is so small, that the distribution defined by the normalizing flow never exceeds the limit defining $A$ or if it is chosen so large, that the volume of $A$ vanishes. 

In this case, we find $\bar A = B$ and $|\bar A| = 1$ where $|\bar A|$ denotes the volume of the data space occupied by $\bar A$. Using the fact that the data distribution has a constant value of $0.9$ in $B$ and that $p_{\theta} < 0.9 - \epsilon$ in $\bar A = B$, 
\begin{align}
    \left| P(\bar A) - P_{\theta}(\bar A) \right| &= \left| 0.9 - P_{\theta}(\bar A)\right|\\
    &\geq \left| 0.9\cdot 1 - (0.9 - \epsilon)\cdot 1\right|\\
    &= \left|\epsilon\right| = \epsilon\label{eq:lower-bound-TV-distance-case-1-counter-example}
\end{align}
Using \cref{eq:lower-bound-TV-distance-case-1-counter-example} as a lower bound for the total variation distance \cref{eq:lower-bound-TV-distance-case-1-counter-example} we can apply \cref{eq:pinskers-inequality} to find \cref{eq:lower-bounf-KL-div-case-1-counter-example} as a lower bound for the KL divergence. 
\begin{align}
    D_{KL}(p||p_{\theta}) &\geq 2 \cdot \delta(p,p_{\theta})^2\label{eq:lower-bounf-KL-div-general-counter-example}\\
    &\geq 2 \cdot \epsilon^2\\
    &= 0.02 \label{eq:lower-bounf-KL-div-case-1-counter-example}
\end{align}

\textbf{Case 2: $A \neq \emptyset$}

Inserting the definition of $p_{\theta}(x,y)$ as given in \cref{eq:change-of-variables} into the definition of $A$ (see \cref{eq:set-A-counter-example}) and rewriting the condition defining the set yields \cref{eq:set-A-rewritten-1-counterexample}.

\begin{equation}
    A = \left\{(x,y)\in\RR^2: \frac{0.9 - \epsilon}{J} < p(z = f_{\theta}(x,y))\right\}\label{eq:set-A-rewritten-1-counterexample}\\
\end{equation}

This defines a set $C$ in the latent space which is defined in \cref{eq:set-C-counter-example}.

\begin{equation}\label{eq:set-C-counter-example}
    C = \left\{z\in\RR^2: \frac{0.9 - \epsilon}{J} < p(z)\right\}
\end{equation}

Since the normalizing flows considered in this example have a constant Jacobian, the volume of $A$ in the data space is directly linked to the volume of $C$ in the latent space via \cref{eq:relation-volume-latent-space-data-space-counter-example}.

\begin{equation}\label{eq:relation-volume-latent-space-data-space-counter-example}
    |A| = \frac{1}{J} \cdot |C|
\end{equation}

The definition of $C$ \cref{eq:set-C-counter-example} shows, that $C$ is a circle around the origin of the latent space. To determine the volume of $C$ we compute the radius of this circle. This is done by inserting the definition of the latent distribution, which is a two-dimensional standard normal distribution, into the condition defining $C$ (see \cref{eq:set-C-counter-example}). This yields \cref{eq:condition-for-radius-latent-space-counter-example}. Since the latent distribution is rotational invariant, one can simply look at it as a function of the distance $r$ from the origin. Solving the for $r$ leads to \cref{eq:radius-counter-example}.

\begin{align}
    &\frac{0.9 - \epsilon}{J} < \frac{1}{2\pi}\cdot \exp\left(-\frac{r^2}{2}\right) \label{eq:condition-for-radius-latent-space-counter-example}\\
    &\Rightarrow r = \sqrt{- 2 \cdot \log\left(\frac{2\pi \cdot (0.9 - \epsilon)}{J}\right)} \label{eq:radius-counter-example}
\end{align}

Inserting \cref{eq:radius-counter-example} into the formula for the area of a circle and using \cref{eq:relation-volume-latent-space-data-space-counter-example}, yields \cref{eq:volume-of-A-counter-example} as an expression for the volume of $A$. The lower bound for the volume of $A$ arises from finding the local maximum (which is also the global maximum) of \cref{eq:volume-of-A-counter-example} with respect to the absolute Jacobian determinant $J$.

\begin{align}
    |A| &= \frac{1}{J} \cdot \pi \cdot r^2\\
    &= \frac{1}{J}\cdot 2\pi\cdot \log\left(\frac{J}{2\pi\cdot(0.9 - \epsilon)}\right)\label{eq:volume-of-A-counter-example}\\
    &\leq \frac{1}{e\cdot (0.9 - \epsilon)}\\
\end{align}

As in the previous case, we now compute $\left| P(\bar A) - P_{\theta}(\bar A) \right|$.

\begin{align}
    \left| P(\bar A) - P_{\theta}(\bar A) \right| &= \left||\bar A| \cdot 0.9 - P_{\theta}(\bar A) \right|\\
    &\geq \left||\bar A| \cdot 0.9 - |\bar A|\cdot (0.9 - \epsilon)  \right|\\
    &= |\bar A|\cdot \epsilon\\
    &\geq \left(1 - |A|\right) \cdot \epsilon\\
    &\geq \left(1 - \frac{1}{e\cdot (0.9 - \epsilon)}\right) \cdot \epsilon
    \label{eq:lower-bounf-tv-dist-case-2-counter-example}
\end{align}

Using \cref{eq:lower-bounf-KL-div-general-counter-example} and \cref{eq:lower-bounf-tv-dist-case-2-counter-example} as a lower bound for the total variation distance and inserting our choice for $\epsilon$ yields \cref{eq:lower-bounf-KL-div-case-2-counter-example} as a lower bound for the KL divergence between the data distribution and the distribution defined by the normalizing flow.

\begin{align}
     D_{KL}(p||p_{\theta}) &\geq 2\cdot \left(\epsilon \cdot \left(1 - \frac{1}{e\cdot (0.9 - \epsilon)}\right)\right)^2\label{eq:lower-bounf-KL-div-case-2-counter-example}\\
     & \approx 0.0058
\end{align}

We can conclude that we have derived lower bounds for the KL divergence between the data distribution and the distribution defined by the normalizing flow, which cannot be undercut by any normalizing flow with a constant absolute Jacobian determinant. Therefore, we have proven that the class of normalizing flows with constant (absolute) Jacobian determinant cannot approximate arbitrary continuous distributions if one uses the KL divergence as a convergence measure.
\end{proof}

\subsection{Proof of \cref{thm:mode-equivalence}}
\label{app:mode-equivalence}

\begin{definition}
    \label{def:mode}
    Given a probability density $p(x)$ and  a connected set $M \subset \mathbb R^D$.
    Then, $M$ is called a \emph{mode} of $p(x)$ if
    \begin{equation}
        p(x) = p(y) \quad \forall x, y \in M,
    \end{equation}
    and there is a neighborhood $U$ of $M$ such that:
    \begin{equation}
        \label{eq:neighborhood-of-mode}
        p(x) > p(y) \quad \forall x \in M, y \in U \setminus M.
    \end{equation}
\end{definition}

With this definition of a mode, let us characterize the correspondence between modes of $p_\theta(x)$ and $p(z)$ for a volume-preserving flow:

\begin{lemma}
    \label{lem:mode-equivalence}
    Given a latent probability density $p(z)$, a diffeomorphism $f_\theta: \mathbb R^D \to \mathbb R^D$  with constant Jacobian determinant $|f_\theta'(x)| = \operatorname{const}$ and a mode $M \subset \mathbb R^D$.
    Then, $f(M)$ is a mode of $p_\theta(x)$.
\end{lemma}
\begin{proof}
    We show that $f(M)$ fulfills  \cref{def:mode}. First, for every $x, y \in f(M)$:
    The pre-images of $x, y$ are unique in $M$ as $f$ is bijective, that is: $f^{-1}(x), f^{-1}(y) \in M$.
    As $M$ is a mode:
    \begin{equation}
        p(f^{-1}(x)) = p(f^{-1}(y)).
    \end{equation}
        We follow:
    \begin{equation}
        (f_\sharp p)(x) = p(f^{-1}(x)) |J| = p(f^{-1}(y)) |J| = (f_\sharp p)(y),
    \end{equation}
    where we have used the change-of-variables formula for bijections and that $|J| = \operatorname{const}$.

    Let $U$ be a neighborhood of $M$ such that \cref{eq:neighborhood-of-mode} is fulfilled.
    As $f$ is continuous, there is a neighborhood $V$ of $f(M)$ such that $V \subseteq f(U)$.
    Consider $x \in f(M), y \in V \setminus f(M)$. As $M$ is a mode:
    \begin{equation}
        p(f^{-1}(x)) > p(f^{-1}(y)).
    \end{equation}
    Multiplying both sides by $|J|$, we find:
    \begin{equation}
        (f_\sharp p)(x) = p(f^{-1}(x))|J| > p(f^{-1}(y))|J| = (f_\sharp p)(y).
    \end{equation}

    Thus, $f(M)$ is a mode of $(f_\sharp p)(x)$ by \cref{def:mode}.
\end{proof}

This makes us ready for the proof:
\begin{repproposition}{thm:mode-equivalence}
    A normalizing flow $p_\theta(x)$ with constant Jacobian determinant $|f_\theta'(x)| = \operatorname{const}$ has the same number of modes as the latent distribution $p(z)$.
\end{repproposition}
\begin{proof}
    By \cref{lem:mode-equivalence}, every mode of $p(x)$ implies a mode of $(f_\sharp p)(x)$.
    Also, every mode of $(f_\sharp p)(x)$ implies a mode of $(f_\sharp^{-1} f_\sharp p)(x) = p(x)$.
    Therefore, there is a one-to-one correspondence of modes between $p(x)$ and $(f_\sharp p)(x)$.
\end{proof}

\section{Proofs on Affine Coupling Flows}

\subsection{Proof of Convergence Measure Minimum}

\subsubsection{Underlying Results from Previous Work}

Here, we restate the results from the literature that our main proof is based on:

First, \citep{eaton1986characterization} show that if for some vector-valued random variable $X$ and every pair of orthogonal projections the mean of one projection conditioned on the other is zero, then $X$ follows a spherical distribution:
\begin{theorem}[\citet{eaton1986characterization}]
    \label{thm:spherical-characterization}
    Suppose the random vector $X \in \RR^D$ has a finite mean vector. Assume that for each vector $v \neq 0$ and for each vector $u$ perpendicular to $v$ (i.e.~$u \cdot v = 0$):
    \eql{
        \EE[u \cdot X | v \cdot X] = 0.
    }
    Then $X$ is spherical and conversely.
\end{theorem}

Secondly, \citet[Corollary 8a]{cambanis1981theory} identifies the Gaussian from all elliptically contoured (which includes spherical) distributions. We write it in the form of \citet[Theorem 4.1.4]{bryc1995normal}:
\begin{theorem}[\citet{bryc1995normal}]
    \label{thm:normal-characterization}
    Let $p(x)$ be radially symmetric with $\EE[\norm{x}^\alpha] < \infty$ for some $\alpha > 0$. If
    \eql{
        \EE[\norm{x_{1, \dots, m}}^\alpha | x_{m+1, \dots, n}] = \const,
    }
    for some $1 \leq m < n$, then $p(x)$ is Gaussian.
\end{theorem}

Finally, \citet[Theorem 1]{draxler2020characterizing} show that the explicit form of the maximally achievable loss improvement by an affine coupling block $\Delta_{\textnormal{affine}}^*(p_\theta(z))$ if the data is rotated by a fixed rotation layer $Q$ is given by (omitting the dependence on $p_\theta(z)$ to avoid clutter):
\begin{align}
    \Delta_{\textnormal{affine}}^*(Q) &
    = \KL{p_0(z)}{p(z)} - \min_{s, t} \KL{p_{s,t|Q}(z)}{p(z)} \\&
    = \tfrac12 \EE_{b}\left[ m_i(b)^2 + \sigma_i(b)^2 - 1 - \log \sigma_i(b)^2 \right].
    \label{eq:improvement-by-rotation}
\end{align}
Here, $s, t$ are the scaling and translation in an affine coupling block (see \cref{eq:affine-coupling}), and we optimize over continuous functions for now. By $p_{s,t|Q}(z)$ we denote the latent distribution achieved if $\tilde a(a; b) = s(b) \odot a + t(b)$ is applied to $p_0(a, b)$, the rotated version of the incoming $p_0(z)$. The symbols $m_i(b), \sigma_i(b)^2$ are conditional moments of the active dimensions $a_i$ conditioned on the passive dimensions $b$:
\begin{equation}
    m_i(b) = \EE_{a_i|b}[a_i], \qquad \sigma_i(b) = \EE_{a_i|b}[a_i^2] - m_i(b)^2.
\end{equation}
These conditional moments are continuous functions of $b$ if $p(x)$ is a continuous distribution and $p(b) > 0$ for all passive $b \in \RR^{D/2}$. The improvement in \cref{eq:improvement-by-rotation} is achieved by the affine coupling block with the following subnetwork:
\begin{equation}
    \label{eq:continuous-minimizing-network}
    s_i^*(b) = \frac{1}{\sigma_i(b)}, \qquad
    t_i^*(b) = -\frac{m_i(b)}{\sigma_i(b)}.
\end{equation}
Note that $s^*(b)$ and $t^*(b)$ are continuous functions and not actual neural networks. In the next section, we show that a similar statement on practically realizable neural networks that is sufficient for our universality.

\subsubsection{Relation to practical neural networks (\cref{lem:convergence-equivalence})}
\label{proof:convergence-equivalence}

Before moving to the proof of \cref{thm:gaussian-identification}, we show the helper statement \cref{lem:convergence-equivalence}.
For reference, let us repeat the definition of the loss improvement by an affine coupling block $\Delta_{\textnormal{affine}}(p_\theta(z))$:
\begin{repdefinition}{def:convergence-measure}
	Given a normalizing flow $f_\theta(x)$ on a continuous probability distribution $p(x)$ with finite first and second moment and $p(x) > 0$ everywhere.
	Then, we define the \emph{loss improvement by an affine coupling block} as:
	\begin{equation}
		\label{repeq:loss-improvement}
			\Delta_{\textnormal{affine}}(p_\theta(z))
			:= \KL{p_{\theta}(z)}{p(z)} - \min_{\theta_+} \KL{p_{\theta \cup \theta_+\!}(z)}{p(z)},
	\end{equation}
	where $\theta_+ = (Q, \phi)$ parameterizes a single $L$-bi-Lipschitz affine coupling block whose conditioner neural network $\psi$ has at least two hidden layers of finite width and ReLU activations.
\end{repdefinition}

To relate \cref{eq:continuous-minimizing-network,eq:improvement-by-rotation} to actually realizable networks, which cannot \textit{exactly} follow the arbitrary continuous functions $s_i^*(b), t_i^*(b)$, the following statement asserts that the fix point of adding coupling layers with infinitely expressive conditioner functions is the same as actually realizable and well-conditioned coupling blocks:
\begin{replemma}{lem:convergence-equivalence}
	Given a continuous probability density $p(z)$ on $z \in \RR^k$. Then,
	\begin{align}
		\Delta_{\textnormal{affine}}^*(p_\theta(z)) &> 0
		\intertext{if and only if:}
		\label{repeq:neural-net-reduces-loss}
		\Delta_{\textnormal{affine}}(p_\theta(z)) &> 0.
	\end{align}
\end{replemma}
\begin{proof}
    First, note that $\Delta_{\textnormal{affine}}^*(Q) \geq \Delta_{\textnormal{affine}}(Q) \geq 0$ since no practically realizable coupling block can achieve better than \cref{eq:improvement-by-rotation}. Thus, if $\Delta_{\textnormal{affine}}^*(Q) = 0$, so is $\Delta_{\textnormal{affine}}(Q) = 0$.

    For the reverse direction, we fix $Q = I$, and otherwise consider a rotated version of $p$. Also, without loss of generalization, we consider one single active dimension $a_i$ in the following, but the construction can then be repeated for each other active dimension.

    If we apply any affine coupling layer $f_{\text{cpl},\phi}(a; b) = s_\phi(b) a + t_\phi(b)$, the loss change by this layer can be computed from the theoretical maximal improvement $\Delta_{\textnormal{affine}}^*(Q)$ before and after adding this layer $\tilde \Delta_{\textnormal{affine}}^*(I)$:
    \begin{equation}
        \Delta_{\textnormal{affine}}(I) 
        = \Delta_{\textnormal{affine}}^*(I) - \tilde \Delta_{\textnormal{affine}}^*(I)
        = \tfrac12 \EE_{b}\left[ m_i(b)^2 + \sigma_i(b)^2 - 1 - \log \sigma_i(b)^2 \right] - \tfrac12 \EE_{b}\left[ \tilde m_i(b)^2 + \tilde \sigma_i(b)^2 - 1 - \log \tilde \sigma_i(b)^2 \right]. 
    \end{equation}

    The moments after the affine coupling layers read:
    \begin{equation}
        \tilde m_i(b) = s_\phi(b) m_i(b) + t_\phi(b), 
        \qquad 
        \tilde \sigma_i(b) = s_\phi(b) \sigma_i(b).
    \end{equation}

    Case 1: $\EE_b[\sigma_i(b)^2 - 1 - \log \sigma_i(b)^2] > 0$:

    Then, without loss of generality, by continuity and positivity of $p$ and consequential continuity of $\sigma_i(b)$ in $b$, there is a convex open set $A \subset \RR^{D/2}$ with non-zero measure $p(A) > 0$ where $\sigma_i(b) > 1$.
    If $\sigma_i(b) < 1$ everywhere, apply the following argument flipped around $\sigma_i(b) = 1$.

    Denote by $\sigma_{\max} = \max_{b \in A} \sigma_i(b)$. Then, by continuity of $\sigma_i(b)$ there exists $B \subset A$ so that $\sigma_i(b) > (\sigma_{\max} - 1) / 2 + 1 =: \sigma_{\max / 2}$ for all $b \in B$. Let $C \subset B$ be a multidimensional interval $[l_1, r_1] \times \dots \times [l_{D/2}, r_{D/2}]$ with $p(C) > 0$ inside $B$.

    Now, we construct a ReLU neural network with two hidden layers with the following property, where $F \subset E \subset C$ are specified later with $p(F) > p(E) > 0$:
    \begin{equation}
        \begin{cases}
            f_\phi(x) = \frac{1}{\sigma_{\max / 2}} & x \in E \subset D \\
            \frac{1}{\sigma_{\max / 2}} \leq f_\phi(x) < 1 & x \in D \\
            f_\phi(x) = 0 & \text{else.}
        \end{cases}
        \label{eq:box-neural-network}
    \end{equation}

    To do so, we make four neurons for each dimension $i = 1, \dots, D/2$:
    \begin{equation}
        \relu(x_i - l_i), \relu(x_i - l_i - \delta), \relu(x_i - r_i), \relu(x_i - r_i + \delta),
    \end{equation}
    where $0 < \delta < \min_i (r_i - l_i) / 4$. If we add these four neurons with weights $1, -1, -1, 1$, we find the following piecewise function:
    \begin{equation}
        \begin{cases}
            0 & x \leq l_i \\
            x - l_i & l_i < x < l_i + \delta \\
            \delta & l_i + \delta \leq x \leq r_i - \delta \\
            r_i - x & r_i - \delta < x < r_i \\
            0 & r_i \leq x.
        \end{cases}
    \end{equation}

    If we repeat this for each dimension and add together all neurons with the corresponding weights into a single neuron in the second layer, then only inside $D = (l_1 + \delta, r_1 - \delta) \times \dots \times (l_{D/2} + \delta, r_{D/2} - \delta) \subset C$ the weighted sum would equal $\delta D/2$. By choosing $\delta$ as above, this region has nonzero volume. We thus equip the single neuron in the second layer with a bias of $-\delta D/2 + \epsilon$ for some $\epsilon < \delta$, so that it is constant with value $\epsilon$ inside $E = (l_1 + \delta - \epsilon, r_1 - \delta + \epsilon) \times \dots \times (l_{D/2} + \delta + \epsilon, r_{D/2} - \delta - \epsilon) \subset D$ and smoothly interpolates to zero in the rest of $D$.

    For the output neuron of our network, we choose weight $(\sigma_{\max / 2} - 1)/\epsilon$ and bias $1$.
    By inserting the above construction, we find the network specified in \cref{eq:box-neural-network}.

    Now, for all $b \in D$,
    \begin{equation}
        1 < \tilde \sigma_i(b) < \sigma_i(b),
    \end{equation}
    so that
    \begin{equation}
        \tilde m_i(b)^2 + \tilde \sigma_i(b)^2 - 1 - \log \tilde \sigma_i(b)^2 < m_i(b)^2 + \sigma_i(b)^2 - 1 - \log \sigma_i(b)^2.
    \end{equation}
    Thus, parameters $\phi$ exist that improve on the loss.
    (Note that this construction can be made more effective in practice by identifying the sets where $\sigma > 1$ resp.~$\sigma < 1$ and then building neural networks that output one or scale towards $\tilde \sigma(b) = 1$ everywhere. Because we are only interested in identifying improvement, the above construction is sufficient.)

    Now, regrading $t_\phi$, we focus on $\EE_{b}[m_i(b)^2] > 0$ (otherwise choose $t_\phi = 0$ as a constant, which corresponds to a ReLU network with all weights and biases set to zero):
    \begin{align}
        \EE_{b}[m_i(b)^2] > \EE_{b}[(s_\phi(b) m_i(b) + t_\phi(b))^2].
    \end{align}
    By \citep[Theorem 1]{hornik1991approximation} there always is a $t_\phi$ that fulfills this relation.

    Case 2: $\EE_{b}[\sigma_i(b)^2 - 1 - \log \sigma_i(b)^2] = 0$. Then, choose the neural network $s_\phi(b) = 1$ as a constant. As $\Delta_{\textnormal{affine}} > 0$, $\EE_{b \sim p(a, b)}[m_i(b)^2] > 0$ and we can use the same argument for the existence of $t_\phi$ as before.

    It is left to show that a $L$-bi-Lipschitz coupling block can be constructed. To achieve this, replace the action of the coupling block $\tilde a_i = s(b) a_i + t(b)$ by $\tilde a_i = \alpha(s(b) a_i + t(b)) + (1 - \alpha) a_i$. Since $s(b)$ and $t(b)$ above were constructed to move the data in the right direction, we obtain a finite loss improvement, since $\alpha > 0 \Leftrightarrow \Delta_{\textnormal{affine}} > 0$. The Jacobian $J$ of the restricted coupling block is $|f_\text{cpl}'| = \alpha |f_\text{original}'| + (1 - \alpha) I$. Since the eigenvectors are unchanged, all eigenvalues $\lambda^{(i)}_\text{original}$ of $|f_\text{original}'|$ are modified to $\lambda_\alpha^{(i)} = \alpha \lambda^{(i)}_\text{original} + (1-\alpha)$. This moves all eigenvalues closer to $1$. Choose $\alpha > 0$ such that $\min_i \lambda_\alpha^{(i)} \geq L^{-1}$ and $\max_i \lambda_\alpha^{(i)} \leq L$ to achieve $L$-bi-Lipschitzness.
\end{proof}

\subsubsection{Proof of \cref{thm:gaussian-identification}}
\label{proof:gaussian-identification}

The following theorem based on $\cref{lem:convergence-equivalence}$ shows that $\Delta_{\textnormal{affine}}(p_\theta(z))$ is a useful measure of convergence to the standard normal distribution:
\begin{reptheorem}{thm:gaussian-identification}
	With the definitions from \cref{def:convergence-measure}:
	\begin{equation}
		p_\theta(z) = \Nn(z; 0, I)
		\Leftrightarrow
		\Delta_{\textnormal{affine}}(p_\theta(z)) = 0.
	\end{equation}
\end{reptheorem}
\begin{proof}
    The forward direction is trivial: $p(z) = \Nn(0, I)$ and therefore $\KL{p(z)}{\Nn(0, I)} = 0$. As adding a identity layer is a viable solution to \cref{eq:minimal-layer-loss}, there is a $\phi$ with $\KL{p_\phi(z)}{\Nn(0, I)} = 0$, and thus $\Delta_{\textnormal{affine}}(p_\theta(z)) = 0$.
    
    For the reverse direction, start with $\Delta_{\textnormal{affine}}(p_\theta(z)) = 0$. Then, by \cref{lem:convergence-equivalence}, also $\Delta_{\textnormal{affine}}^*(p_\theta(z)) = 0$.
    
    The maximally achievable loss improvement for any rotation $Q$ is then given by \cref{eq:best-loss-improvement}:
    \begin{equation}
        \Delta_{\textnormal{affine}}^*(p_\theta(z)) = \max_Q \frac12 \sum_{i=1}^{D/2} \EE_{b}\left[ m_i(b)^2 + \sigma_i(b)^2 - 1 - \log \sigma_i(b)^2 \right] = 0.
    \end{equation}
    It holds that both $x^2 \geq 0$ and $x^2 - 1 - \log x^2 \geq 0$. Thus, the following two summands are zero:
    \begin{align}
        0 &= \tfrac12 \EE_{b}\left[ m_i(b)^2 \right], \\
        0 &= \tfrac12 \EE_{b}\left[ \sigma_i(b)^2 - 1 - \log \sigma_i(b)^2 \right].
    \end{align}
    This holds for all $Q$ since the maximum over $Q$ is zero.

    By continuity of $p(b)$ and $m_1(b)$ in $p$, this implies for all $b$:
    \begin{equation}
        \EE_{a_1|b}[a_1] = 0.
    \end{equation}
    Fix $b_1$ and marginalize out the remaining dimensions $b_{2, \dots D/2}$ to compute the mean of $a_1$ conditioned on $b_1$:
    \begin{equation}
        m_{a_1|b}
        = \EE_{a_1|b_1}[a_1]
        = \EE_{b_{1, \dots, D/2}}[\EE_{a_1|b}[a_1]]
        = \EE_{b_{1, \dots, D/2}}[0]
        = 0.
    \end{equation}
    As $a_1$ and $b_1$ are arbitrary orthogonal directions since the above is valid for any $Q$, we can employ \cref{thm:spherical-characterization} to follow that $p(x)$ is spherically symmetric.
        
    We are left with showing that for a spherical $p(x)$, if for all $Q$ there is no improvement $\Delta_{\textnormal{affine}}(Q)$, then $p(x) = \Nn(0, I)$.
    
    Without loss of generality, we can fix $Q = I$, as $(Q_\sharp p)(x) = p(x)$ for all $Q$. We write $x = (p; a)$.
    
    As $\Delta_{\textnormal{affine}} = 0$, we can follow $\sigma_i(b) = 1$ like above.
    This implies that:
    \eql{
        \EE_{a|b}[\norm{a}^2] = \sum_{i=1}^{D/2} (\mean_i(b)^2 + \sigma_i(b)^2) = D/2.
    }
    In particular, this is independent of $b$ and we can thus apply \cref{thm:normal-characterization} with $\alpha = 2$.

    Finally, $\mean(b) = 0$ and $\sigma_i(b) = 1$ for all $Q$ imply that $p(x) = \Nn(0, I)$.
\end{proof}

\subsection{Proof of \cref{thm:affine-coupling-universality}}
\label{app:proof-affine-coupling-universality}

\begin{figure}
    \centering
    \includegraphics[width=.5\linewidth, trim=0 0 6.5cm 0, clip]{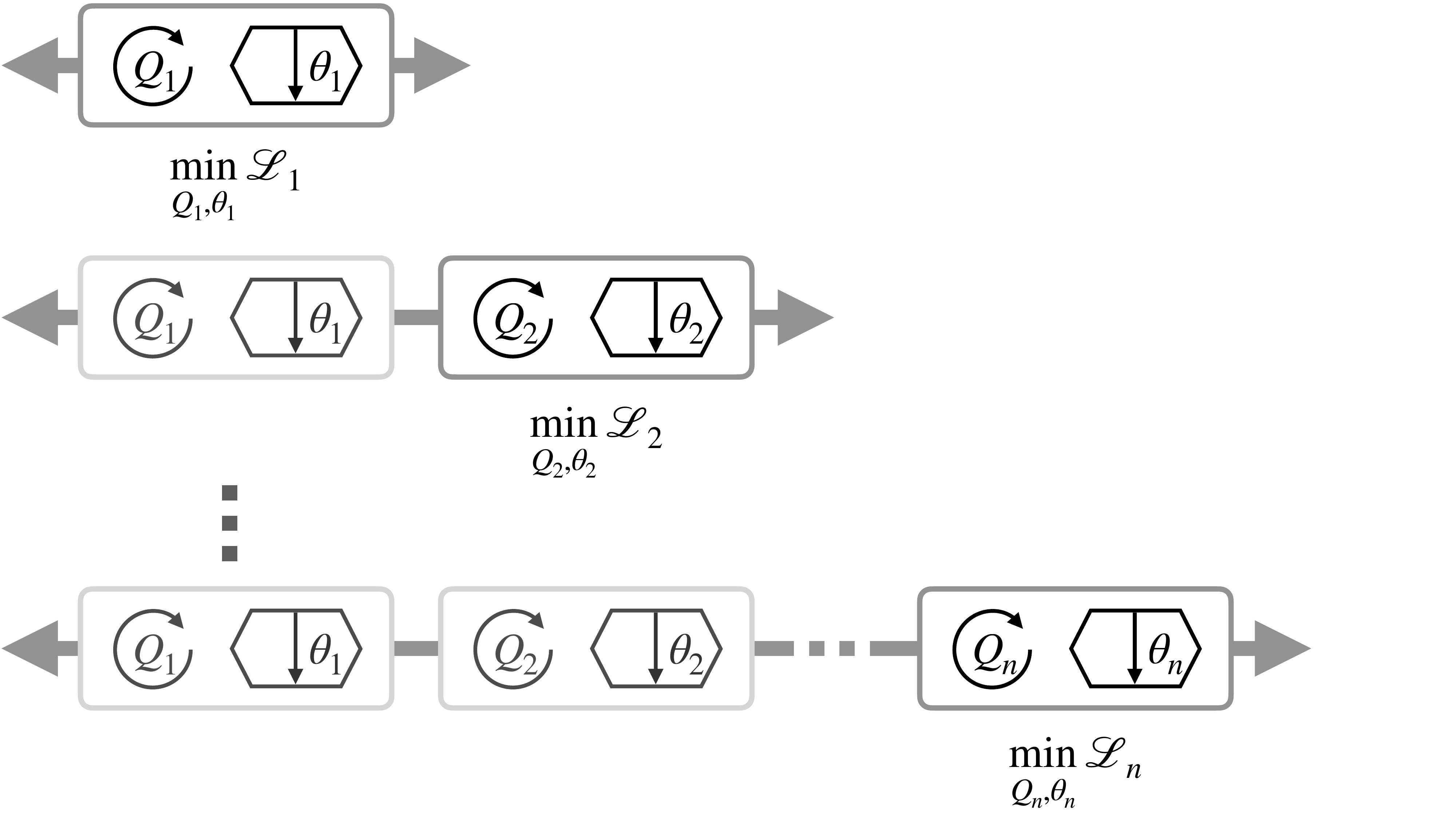}
    \caption{
        The normalizing flow we construct in our proof is remarkably simple: We iteratively add coupling blocks, optimizing the parameters of the new block while keeping previous parameters fixed. \Cref{thm:gaussian-identification} shows that if adding another blocks shows no improvement in the loss, the flow has converged to a standard normal distribution in the latent space. Since the total loss that can be removed is finite, the flow converges.
    }
    \label{fig:universality-proof-idea}
\end{figure}

\begin{reptheorem}{thm:affine-coupling-universality}
	For every continuous $p(x)$ with finite first and second moment with infinite support, there is a sequence of normalizing flows $f_n(x)$ consisting of $n$ affine coupling blocks such that:
	\begin{align}
		p_n(z) \xrightarrow{n \to \infty} \Nn(z; 0, I),
	\end{align}
	in the sense that $\Delta_\textnormal{affine}(p_n(z)) \xrightarrow{n \to \infty} 0$.
\end{reptheorem}
\begin{proof}
	The proof idea of iteratively adding new layers which are trained without changing previous layers is visualized in \cref{fig:universality-proof-idea}.
	
    Let us consider a coupling-based normalizing flow of depth $n$ and call the corresponding latent distributions $p_n(z)$, where $n=0$ corresponds to the initial data distribution $p(x)$. Denote by $\Ll_n = \KL{p_n(z)}{p(z)}$ the corresponding loss. Then, if we add another layer to the flow, we achieve a difference in loss of:
    $\Delta_{{\textnormal{affine}}}(p_n(z)) = \Ll_{n} - \Ll_{n+1}$.

    Without loss of generality, we may assume that the rotation layer $Q$ of each block can be chosen freely. Otherwise, add 48 coupling blocks with fixed rotations that together exactly represent the $Q$ we want, as shown by \citet[Theorem 2]{koehler2021representational}.

    We construct the blocks of the flow iteratively: Choose the rotation and subnetwork parameters $\phi_{n+1}$ of each additional block such that the block maximally reduces the loss, keeping the parameters of the previous blocks $\phi_{1, \dots, n}$ fixed. Then, $\Delta_{{\textnormal{affine}}}(p_n(z))$ attains the value given in \cref{eq:loss-improvement} (\cref{repeq:loss-improvement}):
    \begin{equation}
        \Delta_{\textnormal{affine}}(p_n(z)) = \Ll_n - \Ll_{n+1} = \KL{p_{\phi_{1, \dots, n}}(z)}{p(z)} - \min_{\phi_{n+1}} \KL{p_{\phi_{1, \dots, n} \cup \phi_{n+1}}(z)}{p(z)} \geq 0,
    \end{equation}

    Each layer contributes a non-negative improvement in the loss which can at most sum up to the initial loss:
    \begin{equation}
        \sum_{k=0}^{n-1} \Delta_{{\textnormal{affine}}}(p_k(z)) = \Ll_0 - \Ll_n \leq \Ll_0 \quad \text{for all } n \geq 1,
    \end{equation}
    and the inequality is due to $\Ll_n \geq 0$.
    For a non-negative series that is bounded above, the terms of the series must converge to zero \citep[Theorems 3.14 and 3.23]{rudin1976principles}, which shows convergence in terms of \cref{sec:convergence-metric}:
    \begin{equation}
        \label{eq:finite-series-convergence}
        \sum_{n=0}^\infty \Delta_{{\textnormal{affine}}}(p_n(z)) \leq \Ll_0 < \infty \Rightarrow \Delta_{{\textnormal{affine}}}(p_n(z)) \to 0.
    \end{equation}
\end{proof}

\subsection{Relation to Convergence in KL}
\label{app:kl-implies-improvement}

\begin{corollary}
    \label{corr:kl-implies-improvement}
    Given a series of probability distributions $p_n(z)$. Then, convergence in KL divergence
    \begin{equation}
        \label{eq:kl-convergence}
        \KL{p_n(z)}{\N(0, 1)} \xrightarrow{n \to \infty} 0
    \end{equation}
    implies convergence in the loss improvement by a single affine coupling as in \cref{def:convergence-measure}:
    \begin{equation}
        \Delta_{{\textnormal{affine}}}(p_n(z)) \xrightarrow{n \to \infty} 0.
    \end{equation}
\end{corollary}
\begin{proof}
    By assumption, for every $\epsilon > 0$ there exists $N \in \NN$ such that:
    \begin{equation}
        \KL{p_n(z)}{\N(0, 1)} < \epsilon \quad\forall n > N.
    \end{equation}
    This implies convergence of $\Delta_{{\textnormal{affine}}}(p_n(z))$, by the following upper bound via the sum of all possible future improvements which is bounded from above by the total loss:
    \begin{equation}
        \Delta_{{\textnormal{affine}}}(p_n(z)) \leq \sum_{m=n}^\infty \Delta_{{\textnormal{affine}}}(p_m(z)) \leq \KL{p_n(z)}{\N(0, 1)} < \epsilon \quad\forall n > N.
    \end{equation}
\end{proof}

\section{Benefits of More Expressive Coupling Blocks}
\label{app:pythagorean-split}

To see what the best improvement for an infinite capacity coupling function can ever be, we make use of the following Pythagorean identity combined from variants in \citet{draxler2022whitening, cardoso2003dependence, chen2000gaussianization}:
\begin{equation}
    \label{eq:coupling-pythagorean-identity}
    \Ll
    = \KL{p_\theta(z)}{\Nn(0, I)}
    = P + \EE_{b \sim p(a, b)}\left[D(b) + J(b) + S(b)\right].
\end{equation}
The symbols $P, D(b), J(b), S(b)$ all denote KL divergences:

The first two terms remain unchanged under a coupling layer: The KL divergence to the standard normal in the passive dimensions $P = \KL{p_\theta(b)}{\N(0, I_{D/2})}$, which are left unchanged. The \textit{dependence} between active dimensions $D(b) = \KL{p_\theta(a|b)}{p_\theta(a_1|b) \cdots p_\theta(a_{D/2}|b)}$ measures the multivariate mutual information between active dimensions. It is unchanged because each dimension $a_i$ is treated conditionally independent of the others \cite{chen2000gaussianization}.

The remaining terms measure how far each dimension $p_\theta(a_i|b)$ differs from the standard normal: The negentropy measures the divergence to the Gaussian with the same first moments as $p_\theta(a_i|b)$ in each dimension, summing to $J(b) = \sum_{i=1}^{D/2} \KL{p_\theta(a_i|b)}{\Nn(m_i(b), \sigma_i(b))}$. Finally, the non-Standardness $S(b) = \sum_{i=1}^{D/2} \KL{\Nn(m_i(b), \sigma_i(b))}{\Nn(0, 1)}$ measures how far these 1d Gaussian are away from the standard normal distribution.

Note that the total loss $\Ll$ is invariant under a rotation of the data. The rotation does, however, affect how that loss is distributed into the different components in \cref{eq:coupling-pythagorean-identity}.

If we restrict the coupling function to be affine-linear $c(a_i; \theta) = sa_i + t$ (i.e.~a RealNVP coupling), then this means that also $J(b)$ is left unchanged, essentially because $p_\theta(a_i|b)$ and $\Nn(m_i(b), \sigma_i(b))$ undergo the same transformation \citep[Lemma 1]{draxler2022whitening}. Only a nonlinear coupling function $c(a_i; \theta)$ can thus affect $J(b)$ and reduce it to $\tilde J(b) < J(b)$ (if $J(b) > 0$).

Taking the loss difference between two layers, we find \cref{eq:improvement-split}.

\section{Experimental Details}
\label{app:experimental-details}

We base our code on PyTorch \cite{paszke2019pytorch}, Numpy \citep{harris2020array}, Matplotlib \citep{hunter2007matplotlib} for plotting and Pandas \citep{mckinney2010data,reback2020pandas} for data evaluation.

\subsection{Layer-Wise Flow}
\label{app:details-layer-wise-experiment}

In experiment on a toy dataset for \cref{fig:universality-example}, we demonstrate that a coupling flow constructed layer by layer as in \cref{eq:action-single-layer} learns a target distribution. We proceed as follows:

We construct a data distribution on a circle as a Gaussian mixture of $M$ Gaussians with means $m_i = (r \cos \phi_i, r \sin \phi_i)$, where $\phi_i = 0, \frac1M 2\pi, \dots, \frac{M-1}{M} 2\pi $ are equally spaced, and $\sigma_i = 0.3$. The advantage of approximating the ring with this construction is that this yields a simple to evaluate data density, which we need for accurately plotting $p_\theta(z)$:
\begin{equation}
    p(x) = \frac1M \sum_{i=1}^M \Nn(x; m_i, \sigma^2 I).
\end{equation}

We then fit a total 100 layers in the following way: First, treat $p(x)$ as the initial guess for the latent distribution. Then, we build the affine coupling block that maximally reduces the loss using \cref{eq:action-single-layer}. We therefore need to know the conditional mean $m(b)$ and standard deviation $\sigma(b)$ for each $b$. We approximate this from a finite number of samples $N$ which are grouped by the passive coordinate $b$ into $B$ bins so that $N/B$ samples are in each bin. We then compute the empirical mean $m_i$ and standard deviation $\sigma_i$ over the active dimension in each bin $i = 1, \dots, B$. According to \cref{eq:action-single-layer}, we define $s_i = \frac{1}{\sigma_i}$ and $t_i = -\frac{1}{\sigma_i} m_i$ at the bin centers and interpolate between bins using a cubic spline. Outside the domain of the splines, we extrapolate with constants $s, t$ with the value of the closest bin.
We do not directly optimize over $Q$, but choose the $Q$ that reduces the loss most out of $N_Q$ random 2d rotation matrices.

We limit the step size of each layer to avoid artifacts from finite training data, by mapping:
\begin{equation}
    \tilde x = \alpha x + (1 - \alpha) f_\text{blk}(x).
\end{equation}
In addition, we resample the training data from the ground truth distribution after every step to avoid overfitting. We do not explicitly control for the bi-Lipschitz constant of our coupling blocks because we do not encounter any numerical problems.

We choose $N = 2^{26}$, $B=64$, $M=20$, $\alpha=0.5$, $N_Q=10$. The resulting flow has $64 \cdot 2 \cdot 100 = 12,800$ learnable parameters. \Cref{fig:kl-by-layer} shows how the KL divergence vanishes for our layer-wise training, together with $\Delta_\text{affine}$.

\begin{figure}
    \centering
	\centering
    \includegraphics[width=.49\linewidth]{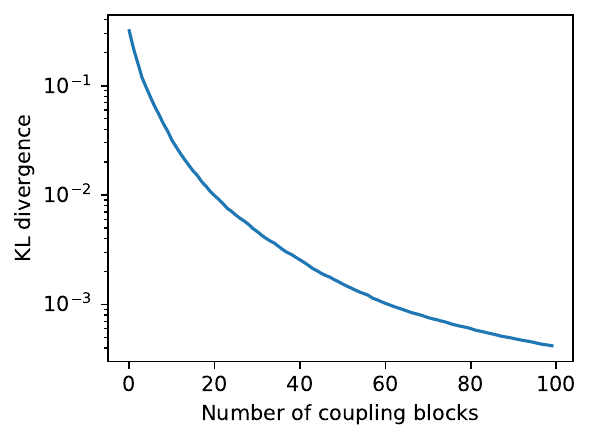}
	\includegraphics[width=.49\linewidth]{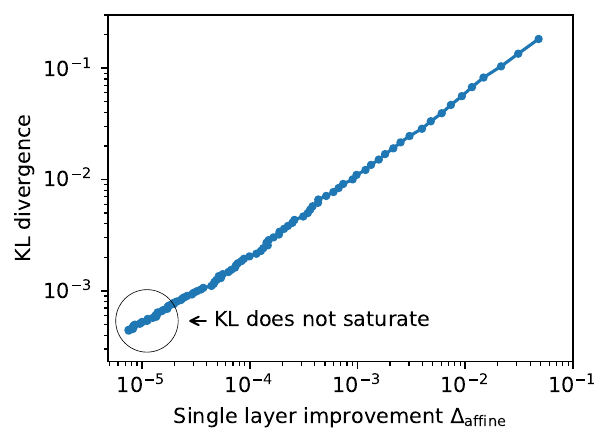}
	\caption{Empirically, the KL divergence decreases as more coupling blocks are added for the toy distribution considered in \cref{fig:universality-example} \textit{(left)}. At the same time, there is a strong correlation between the loss improvement by a single coupling block $\Delta_\textnormal{affine}(p_\theta(z))$ and the KL divergence \textit{(right)}. Crucially, the KL divergence does not saturate as the loss improvements become smaller. The coupling flow considered is trained according to the greedy layer-wise training in our proof construction.}
	\label{fig:layer-improvement-vs-kl}
    \label{fig:kl-by-layer}
\end{figure}

\subsection{Volume-Preserving Normalizing Flows}
\label{app:details-volume-preserving-experiment}

The target distribution is a two-dimensional Gaussian Mixture Model with two modes. The two modes have the same relative weight but different covariance matrices ($\Sigma_1 = I  \cdot 0.2$, $\Sigma_2 = I \cdot 0.1$) and means ($m_1 = [-0.5,-0.5]$, $m_2 = [0.5,0.5]$).

The normalizing flow with a constant Jacobian determinant consists of 15 GIN coupling blocks as introduced in \citet{sorrenson2019disentanglement}. This type of coupling blocks has a Jacobian determinant of one. To allow for a global volume change, a layer with a learnable global scaling is added after the final coupling block. This learnable weight is initialized as one. For the normalizing flow with variable Jacobian determinant, the GIN coupling is modified by removing the normalization of the scaling factors in the affine couplings. This allows the normalizing flow to have variable Jacobian determinants. In this case, the global scaling block is omitted. To implement the normalizing flow, we use the FrEIA package \cite{ardizzone2018framework} implementation of the GIN coupling blocks.

In both normalizing flows, the two subnetworks used to compute the parameters of the affine couplings are fully connected neural networks with two hidden layers and a hidden dimensionality of 128. ReLU activations are used. The weights of the linear layers of the subnetworks are initialized by applying the PyTorch implementation of the Xavier initialization \cite{glorot2010understanding}. In addition, the weights and biases of the final layer of each subnetworks are set to zero.

The networks are trained using the Adam \cite{kingma2017adam} with PyTorch's default settings and an initial learning rate of $1\cdot10^{-3}$ which is reduced by a factor of ten after $5000,10000$ and $15000$ training iterations. In total, the training ran for $25000$ iterations. In each iteration, a batch of size $128$ was drawn from the target distribution to compute the negative log likelihood objective. We use a standard normal distribution as the latent distribution.

For obtaining the optimal distribution $p^*(x)$, we follow the grid procedure in \cref{sec:limitations-vol-preserving-flows} and compute the probabilities on a regular $400 \times 400$ grid of grid spacing $0.01$. The covariance of $p^*(z)$ is computed for the latent scaling layer by sampling $4096$ points from the mixture model, moving them according to the volume-preserving flow learned using the grid and computing their empirical covariance matrix. This yields essentially the same scaling as obtained from training the volume-preserving flow.

\end{document}

%% file: symbols.tex
\renewcommand{\d}{d}

\newcommand{\NN}{\mathbb{N}}

\newcommand{\GG}{\mathbb{G}}

\newcommand{\RR}{\mathbb{R}}

\newcommand{\EE}{\mathbb{E}}

\newcommand{\Aa}{\mathcal{A}}

\newcommand{\Dd}{\mathcal{D}}

\newcommand{\Hh}{\mathcal{H}}

\newcommand{\Ll}{\mathcal{L}}
\newcommand{\Mm}{\mathcal{M}}
\newcommand{\Nn}{\mathcal{N}}

\newcommand{\Pp}{\mathcal{P}}

\newcommand{\Tt}{\mathcal{T}}
\newcommand{\Uu}{\mathcal{U}}

\newcommand{\Xx}{\mathcal{X}}
\newcommand{\Yy}{\mathcal{Y}}
\newcommand{\Zz}{\mathcal{Z}}

\newcommand{\Si}{\Sigma}

\renewcommand{\phi}{\varphi}

\makeatletter
\newcommand{\subalign}[1]{%
  \vcenter{%
    \Let@ \restore@math@cr \default@tag
    \baselineskip\fontdimen10 \scriptfont\tw@
    \advance\baselineskip\fontdimen12 \scriptfont\tw@
    \lineskip\thr@@\fontdimen8 \scriptfont\thr@@
    \lineskiplimit\lineskip
    \ialign{\hfil$\m@th\scriptstyle##$&$\m@th\scriptstyle{}##$\hfil\crcr
      #1\crcr
    }%
  }%
}
\makeatother

\newcommand{\eql}[1]{\begin{equation}#1\end{equation}}

\newcommand{\transy}{\mathrm{T}}

\newcommand{\norm}[1]{\left\| #1 \right\|}

\let\originalleft\left
\let\originalright\right
\renewcommand{\left}{\mathopen{}\mathclose\bgroup\originalleft}
\renewcommand{\right}{\aftergroup\egroup\originalright}

\DeclareMathOperator{\relu}{ReLU}
\DeclareMathOperator{\Var}{Var}

\DeclareMathOperator{\const}{const}
\newcommand{\KL}[2]{\Dd_\text{KL}(#1 \| #2)}

\newcommand{\mean}{m}

\newcommand{\rot}{Q}
\newcommand{\mat}{}
\renewcommand{\vec}[1]{#1}

\newcommand{\td}{p}

\newcommand{\N}{\Nn}